\documentclass{article}

\usepackage{microtype}
\usepackage{graphicx}
\usepackage{titletoc}
\usepackage{subcaption}
\usepackage{booktabs} % for professional tables
\usepackage{enumitem}
\usepackage{hyperref}
\usepackage{amsmath}
\usepackage{amssymb}

\usepackage{array}
\usepackage{multirow}
\usepackage{adjustbox}
\usepackage{pifont}
\usepackage{tabularx}
\usepackage{xcolor}
\usepackage{colortbl}

\usepackage{float}

\definecolor{scjepagreen}{RGB}{220,238,255}

\usepackage[preprint]{icml2026}

\usepackage{mathtools}
\usepackage{amsthm}

\usepackage[capitalize,noabbrev]{cleveref}

\theoremstyle{plain}
\newtheorem{theorem}{Theorem}[section]

\newtheorem{lemma}[theorem]{Lemma}

\theoremstyle{definition}
\newtheorem{definition}[theorem]{Definition}
\newtheorem{assumption}[theorem]{Assumption}
\theoremstyle{remark}
\newtheorem{remark}[theorem]{Remark}

\usepackage[disable]{todonotes}

\icmltitlerunning{SC-JEPA: Stabilizing Latent Predictive Learning for Time-Series Anomaly Prediction}

\begin{document}

\twocolumn[

  \icmltitle{SC-JEPA: Stabilizing Latent Predictive Learning for Time-Series Anomaly Prediction}
  % It is OKAY to include author information, even for blind submissions: the
  % style file will automatically remove it for you unless you've provided
  % the [accepted] option to the icml2026 package.

  % List of affiliations: The first argument should be a (short) identifier you
  % will use later to specify author affiliations Academic affiliations
  % should list Department, University, City, Region, Country Industry
  % affiliations should list Company, City, Region, Country

  % You can specify symbols, otherwise they are numbered in order. Ideally, you
  % should not use this facility. Affiliations will be numbered in order of
  % appearance and this is the preferred way.
  \icmlsetsymbol{equal}{*}

  \begin{icmlauthorlist}
    \icmlauthor{Yanan He}{hyn}
    \icmlauthor{Yunshi Wen}{rpi}
    \icmlauthor{Xin Wang}{sbu}
   
    \icmlauthor{Tengfei Ma}{sbu}
    %\icmlauthor{}{sch}
  \end{icmlauthorlist}

\icmlaffiliation{hyn}{
Yale University
}

\icmlaffiliation{rpi}{
Rensselaer Polytechnic Institute
}

\icmlaffiliation{sbu}{
Stony Brook University
}

  %\icmlcorrespondingauthor{Firstname1 Lastname1}{first1.last1@xxx.edu}
  \icmlcorrespondingauthor{Tengfei Ma}{tengfei.ma@stonybrook.edu}

  % You may provide any keywords that you find helpful for describing your
  % paper; these are used to populate the "keywords" metadata in the PDF but
  % will not be shown in the document
  \icmlkeywords{Machine Learning, ICML}

  \vskip 0.3in
]

% this must go after the closing bracket ] following \twocolumn[ ...

% This command actually creates the footnote in the first column listing the
% affiliations and the copyright notice. The command takes one argument, which
% is text to display at the start of the footnote. The \icmlEqualContribution
% command is standard text for equal contribution. Remove it (just {}) if you
% do not need this facility.

% Use ONE of the following lines. DO NOT remove the command.
% If you have no special notice, KEEP empty braces:
\printAffiliationsAndNotice{}  % no special notice (required even if empty)
% Or, if applicable, use the standard equal contribution text:
% \printAffiliationsAndNotice{\icmlEqualContribution}

\begin{abstract}
Time-series anomaly prediction aims to forecast future system failures before they fully emerge, making latent predictive models such as JEPA a promising framework for capturing precursor dynamics. However, directly applying continuous self-distillation to time-series data is often unstable and can lead to representation collapse, while also struggling to model precursors evolving at different temporal scales. To address this, we propose \textbf{SC-JEPA}, a new JEPA-based framework to model time-series anomaly prediction in a discretized predictive state space. It introduces a soft codebook bottleneck to stabilize latent predictive learning and encourage regime-level structure in the learned representations. Building on this stabilized latent space, we further design a multi-resolution predictive objective to capture precursor patterns at different temporal scales. Experiments on five real-world benchmarks show that SC-JEPA achieves strong and consistent early-warning performance. The source code is publicly available at \url{https://github.com/Echoo113/SC-JEPA}.
\end{abstract}
\section{Introduction}

Multivariate time series underpin many safety-critical systems, from industrial plants to server operations, where failures can propagate rapidly and incur substantial costs. In such settings, the key challenge is not merely detecting anomalies after they occur, but anticipating them early enough for intervention. This task, known as \textit{anomaly prediction}, differs fundamentally from both anomaly detection and forecasting \cite{park2025fail}. Detection is inherently reactive and identifies faults only after they manifest \cite{deepsvdd, xu2022anomaly}, while forecasting focuses on regressing future observations or values \cite{zhou2021informer, patchtst}. Anomaly prediction, however, aims to isolate latent precursor signals within the current window that foreshadow impending instability \cite{jhin2023precursor,zhao2024abnormalityforecasting,park2025fail}. Consequently, the primary objective is capturing the evolving system risk rather than accurately reproducing the exact future trajectory.

A natural approach to this problem is to model how system states evolve over time. Joint-Embedding Predictive Architectures (JEPA) \cite{i-jepa} offer a compelling paradigm by learning to predict future representations directly in latent space. This objective aligns well with the mechanics of anomaly prediction: impending failures are often characterized by subtle regime shifts and semantic state drifts, rather than deterministic variations in raw signal amplitude. By optimizing  latent predictability, JEPA provides an inherently task-aligned framework for anticipating system instability \cite{ts-jepa}.

\begin{figure}[t]
    \centering
    \begin{subfigure}[t]{0.70\linewidth}
        \centering
        \includegraphics[width=\linewidth]{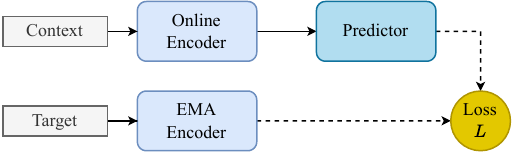}
        \caption{Standard Single-Resolution JEPA}
        \label{fig:single}
    \end{subfigure}

    \vspace{0.5em}

    \begin{subfigure}[t]{0.99\linewidth}
        \centering
        \includegraphics[width=\linewidth]{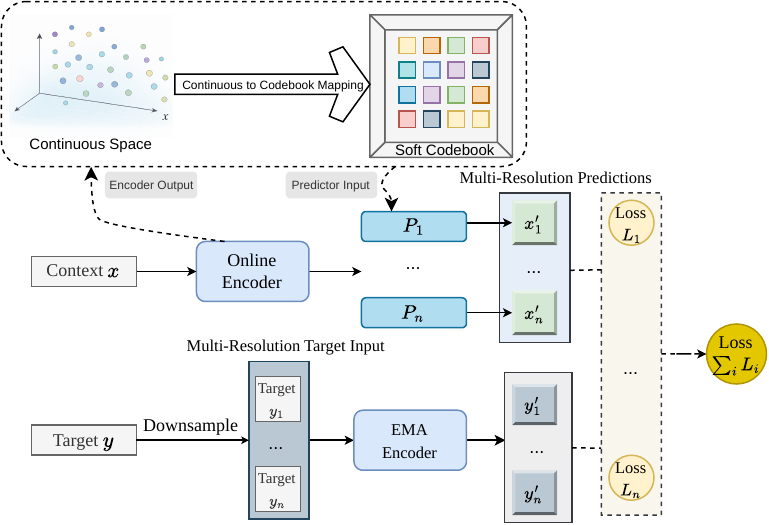}
        \caption{Proposed SC-JEPA}
        \label{fig:multi}
    \end{subfigure}

    \caption{Comparison between the standard single-resolution JEPA and our proposed multi-resolution JEPA with a soft codebook.}
    \label{fig:jepa_comparison}
    \vspace{-2mm}
\end{figure}

However, directly applying JEPA-style predictive learning to anomaly prediction faces two critical bottlenecks. The primary obstacle is a severe vulnerability to representation collapse \cite{simsiam, byol, jing2021understanding, bardes2021vicreg}. First, learning in a continuous latent space without explicit constraints can lead to representation collapse or instability. The encoder may map diverse inputs to a constant vector or experience partial dimensional collapse, stripping the feature space of discriminative capacity \cite{directpred}. The continuous and noisy nature of time series data further exacerbates this instability. Moreover, in real-world systems, anomaly precursors typically emerge across heterogeneous time scales, ranging from abrupt shocks to slowly accumulating drifts \cite{timemixer, scinet, chen2024pathformer, wang2023micn}. A rigid single-scale predictive design (Figure \ref{fig:single}) therefore struggles to capture the full spectrum of precursor dynamics that precede anomalous events. In this work, we argue that these challenges stem from a common source: the lack of structure in the predictive latent space. When representations are unconstrained, the model must simultaneously learn stable dynamics, meaningful abstractions, and scale-aware patterns, leading to degenerate solutions.

To address these challenges, we propose \textbf{SC-JEPA} (\textbf{S}oft \textbf{C}odebook \textbf{JEPA}), a framework that introduces a discretized predictive state space for time-series anomaly prediction. The key idea of SC-JEPA is to constrain latent prediction through a \textbf{\textit{soft codebook bottleneck}}, which maps continuous latent features to distributions over a finite set of prototypes. This mechanism serves two purposes. (i) This design imposes a structured geometry on the latent space, which empirically stabilizes training and encourages representations to organize into regime-level states. (ii) Such structure is well suited to anomaly prediction, where failures often correspond to transitions between underlying system regimes.
Building on this stabilized and more structured latent space, we further introduce a \textit{\textbf{multi-resolution}} predictive design (Figure~\ref{fig:multi}) to model system dynamics over different temporal scales. By jointly predicting fine-grained and coarse-grained future states, the model captures both short-term fluctuations and long-term trends, enabling more robust identification of precursor patterns.

In summary, our main contributions are as follows:
\begin{itemize}[leftmargin=*]
\vspace{-2mm}
\item \textbf{Latent Predictive Learning for Anomaly Prediction:}
We formulate time-series anomaly prediction as a representation-space predictive modeling task, showing that latent state evolution is more informative for precursor identification than raw-value forecasting.

\item \textbf{Soft Codebook as a Structural Stabilizer:} We introduce a soft codebook bottleneck that discretizes latent predictions, encouraging regime-level representations and mitigating collapse in practice.

\item \textbf{Multi-Resolution Extension:} Building on the stabilized latent space induced by the soft codebook, we introduce a multi-resolution predictive objective to better capture precursor signals across temporal scales, thereby enabling earlier anomaly anticipation.
\end{itemize}

\section{Related Works}

\paragraph{Anomaly Prediction.}
Time-series research has largely followed two directions: forecasting and detection. 
Forecasting approaches primarily learn trajectory regression objectives~\cite{zhou2021informer,itransformer} and are typically evaluated on value prediction accuracy rather than early-warning separability. In contrast, detection methods score the current window and are benchmarked under reactive protocols~\cite{deepsvdd,xu2022anomaly}. 
Consequently, there is growing interest in moving beyond reactive detection toward early-warning settings, where the goal is to anticipate future abnormality or failure risk ahead of time~\cite{jhin2023precursor,zhao2024abnormalityforecasting,park2025fail}. However, many existing formulations remain tied to input-space objectives or detector adaptations, suggesting the value of exploring representation-space predictive models that emphasize latent dynamics over local noise, especially when precursors unfold across multiple time horizons. Another line of work related to our method is vector-quantized (VQ) representation learning~\cite{vqvae}, which has been adapted for time-series anomaly detection~\cite{lee2023timevqvae,lee2024timevqvae_ad,wen2025abstracted}; however, our method differs essentially in the target task, representation granularity, and predictive framework.
\paragraph{Joint Embedding Predictive Architectures.}
\label{jepa}
Self-supervised learning for time series is often built on masked modeling, as in PatchTST~\cite{patchtst}, or contrastive learning over augmented views, as in TS2Vec~\cite{ts2vec}. In parallel, Joint-Embedding Predictive Architectures (JEPA) formulate predictive world models by forecasting future representations directly in latent space, as demonstrated by I-JEPA for images and V-JEPA for videos~\cite{i-jepa,v-jepa}. By emphasizing latent predictability, JEPA encourages higher-level, state-like representations, which aligns naturally with time series as noisy observations of underlying evolving processes. Recent efforts have begun to adapt JEPA to time series, including TS-JEPA~\cite{ts-jepa}. 
However, JEPA-style predictive world models remain less explored for complex time-series systems, and stabilizing latent-space self-distillation objectives on continuous, non-stationary time series remains practically challenging.

\paragraph{Multi-Scale Dynamics.}
Multi-resolution modeling is widely used in representation learning to capture both global context and local detail, especially in computer vision~\cite{fpn}.
Time series share a similar multi-scale nature, where long-term trends coexist with fast local variations.
Recent time-series models such as Pyraformer~\cite{pyraformer} and TimeMixer~\cite{timemixer} incorporate multi-scale structure through pyramid-style attention or multi-granularity mixing.
In most cases, scale is introduced to improve forecasting accuracy or efficiency, rather than to explicitly model state evolution at multiple temporal rates.
A less explored direction is to incorporate multi-scale structure into predictive world models, so that latent states can be propagated across different time scales.
This perspective is particularly relevant for anomaly prediction, where early signs may appear as slow drift or short-lived shocks.
These observations point to multi-scale latent dynamics as a natural ingredient for early warning.

\section{Method}

\begin{figure*}[t]
    \centering
    \includegraphics[width=\textwidth]{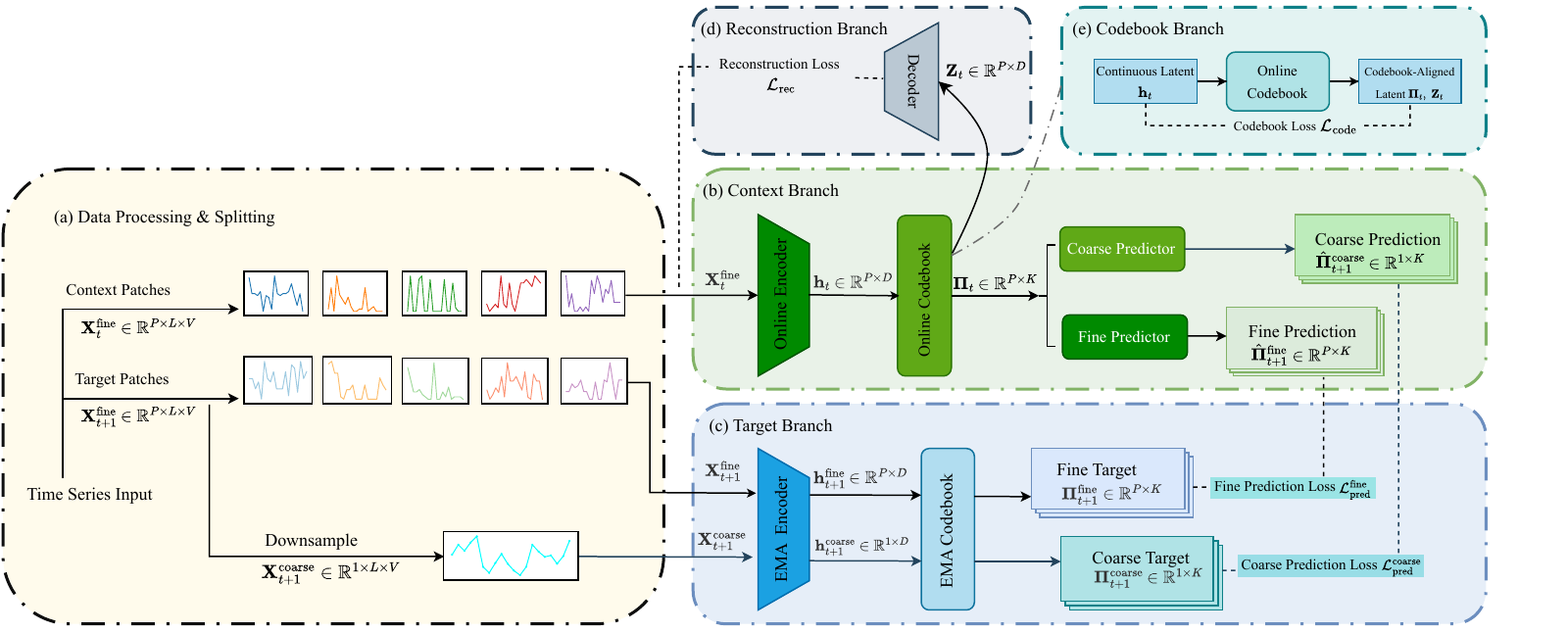}
    \caption{Overview of the SC-JEPA framework for time series representation learning.}
    \label{main:framework}
\end{figure*}

\subsection{Problem Setup}
Consider a multivariate time series $\mathbf{x} \in \mathbb{R}^{T \times V}$.
We partition $\mathbf{x}$ into a sequence of non-overlapping windows $\{\mathbf{W}_t\}_{t=1}^{N}$, where each window $\mathbf{W}_t \in \mathbb{R}^{T_w \times V}$ contains $T_w$ time steps.
Our goal is to learn a representation from the current window $\mathbf{W}_t$ to anticipate the anomaly status of the subsequent window $\mathbf{W}_{t+1}$, enabling proactive intervention by forecasting whether the next window will be anomalous.

\paragraph{Framework Overview.}
This proactive setting is difficult in practice because precursor cues are often weak, non-stationary, and may emerge at different temporal scales, making stable self-supervised prediction nontrivial. We propose a new JEPA framework to solve this challenging task.
(1) To expose such scale-variant precursors, we construct dual-resolution inputs that disentangle local volatilities from global trends, described next in \Cref{sec:input_formulation}.
As illustrated in Figure~\ref{main:framework}, SC-JEPA then adopts an asymmetric teacher--student design: an online context branch predicts future latent states from the current window, while a momentum-updated target branch provides supervision.
The asymmetry imposes a deliberate information gap: the online branch observes only the fine-grained view, yet is supervised to match multi-scale targets produced from both fine and coarse views, encouraging scale-aware representations.
(2) In addition, a soft codebook bottleneck is introduced to impose a discrete inductive bias for anomaly prediction by mapping continuous features to prototype-anchored code distributions, as formalized in \Cref{archi}.

\subsection{Input Formulation and Multi-Scale Views}
\label{sec:input_formulation}
To make scale-variant precursor patterns learnable under a unified predictive objective, we build dual-scale views from each raw window. We first apply RevIN~\cite{kim2022reversible} to $\mathbf{W}_t$, yielding the normalized window $\hat{\mathbf{W}}_t$. The affine statistics are cached to invert normalization for the reconstruction loss.
We then construct two views at different resolutions:

\paragraph{Fine View (Patching).}
We tokenize $\hat{\mathbf{W}}_t$ into $P$ patches of length $L$ (where $T_w = P \cdot L$), forming the fine-grained representation:
\begin{equation}
\mathbf{X}^{\text{fine}}_t = \operatorname{Patch}(\hat{\mathbf{W}}_t) \in \mathbb{R}^{P \times L \times V}.
\end{equation}

\paragraph{Coarse View (Downsampling).}
Simultaneously, we generate a global view by averaging every $P$ time points (Appendix~\ref{app:downavg}) to capture low-frequency trends.
The resulting sequence of length $L$ is treated as a single patch, aligning its dimensionality with the patch-based format:
\begin{equation}
\mathbf{X}^{\text{coarse}}_t = \operatorname{DownAvg}_{P}(\hat{\mathbf{W}}_t) \in \mathbb{R}^{1 \times L \times V}.
\end{equation}
The Online Branch processes only $\mathbf{X}^{\text{fine}}_t$, while the Target Branch encodes both resolutions at $t{+}1$ for supervision.
This asymmetry serves as a structural inductive bias, compelling the model to encapsulate multi-scale semantics. 
Consequently, to bridge this deliberate information gap, we next detail the architectural machinery constructed to orchestrate the flow from partial local observations to stable global predictions.

\subsection{Architecture}
\label{archi}
SC-JEPA couples a shared encoder, a soft codebook, a dual-resolution predictor, and an auxiliary decoder.

\paragraph{Encoder.}
Weak and scale-variant precursors call for representations that are predictable and comparable across resolutions.
We therefore use a shared encoder $\mathcal{E}$ to embed raw windows into a compact latent space, using a channel-independent input formulation.
Implemented as a residual CNN tokenizer followed by a Transformer backbone, $\mathcal{E}$ supports variable-length inputs, so both fine-grained patch sequences and a single coarse-grained token are processed with the same parameters.
Thus, a single encoder maps both resolutions into the same embedding space.

\begin{enumerate}[label=(\roman*), leftmargin=*]
\item \textbf{Online encoder ($\mathcal{E}_\theta$):}
$\mathbf{h}_t=\mathcal{E}_\theta(\mathbf{X}^{\text{fine}}_t)$.
Operating under partial observation, $\mathcal{E}_\theta$ receives only the fine-grained view. It is trained to align with targets produced from both fine and coarse views (via the momentum teacher), which pressures the online representation to internalize global context from local evidence.

\item \textbf{EMA encoder ($\mathcal{E}_\xi$):}
A momentum-updated copy of the online encoder with holistic visibility. It processes both fine and coarse views to generate comprehensive, multi-scale targets:
$\mathbf{h}^{\text{fine}}_{t+1}=\mathcal{E}_\xi(\mathbf{X}^{\text{fine}}_{t+1})$ and
$\mathbf{h}^{\text{coarse}}_{t+1}=\mathcal{E}_\xi(\mathbf{X}^{\text{coarse}}_{t+1})$.
\end{enumerate}
% [Hook to Codebook]

However, continuous encoder features often entangle nuisance variations.
To mitigate this, we introduce a soft codebook that projects features onto a finite set of prototypes.
This discretization not only captures regime-like patterns but also imposes an explicit geometric constraint on the latent space, which, as we detail next, is crucial for ensuring both representation stability and expressiveness.

\paragraph{Codebook.}
Formally, we implement this projection via a differentiable bottleneck $\mathcal{Q}$ that maps continuous encoder features into a finite set of learnable prototypes $\{\mathbf{c}_k\}_{k=1}^{K} \subset \mathbb{R}^{D}$.
For each patch representation $\mathbf{h}_{t,i}\in\mathbb{R}^{D}$, $\mathcal{Q}$ outputs a soft code distribution
$\mathbf{p}_{t,i}\in\Delta^{K-1}$, where $\Delta^{K-1}\coloneqq\{\mathbf{p}\in\mathbb{R}^K:\mathbf{p}\succeq 0,\ \mathbf{1}^\top\mathbf{p}=1\}$.
We parameterize $\mathbf{p}_{t,i}$ via temperature-scaled cosine similarity:
\begin{equation}
p_{t,i,k} \;=\;
\frac{\exp\!\left\{\left\langle \bar{\mathbf{h}}_{t,i},\, \bar{\mathbf{c}}_{k}\right\rangle/\tau\right\}}
{\sum_{j=1}^{K}\exp\!\left\{\left\langle \bar{\mathbf{h}}_{t,i},\, \bar{\mathbf{c}}_{j}\right\rangle/\tau\right\}},
\label{eq:softassign}
\end{equation}
where $\bar{\mathbf{h}}_{t,i}$ and $\bar{\mathbf{c}}_{k}$ denote $\ell_2$-normalized vectors.
We stack these patch-wise distributions to form the window-level code sequence
$\boldsymbol{\Pi}_t \coloneqq [\mathbf{p}_{t,1}; \dots; \mathbf{p}_{t,P}] \in \mathbb{R}^{P \times K}$,
which serves as the predictor input.
In parallel, we compute the expected code embedding
$\mathbf{z}_{t,i} \coloneqq \sum_{k=1}^{K} p_{t,i,k} \mathbf{c}_k \in \mathbb{R}^{D}$,
which is used by the auxiliary decoder for reconstruction.

Since $\mathbf{p}_{t,i}\in\Delta^{K-1}$, the soft embedding $\mathbf{z}_{t,i}=\sum_{k}p_{t,i,k}\mathbf{c}_k$ constitutes a convex combination of finitely many prototypes and therefore remains within the bounded convex hull of $\{\mathbf{c}_k\}$.
This boundedness provides an explicit geometric constraint on the latent space, preventing unbounded feature excursions and improving score stability; a formal \textbf{upper bound} is given in Appendix~\ref{upper-bound}. 
Complementarily, Appendix~\ref{app:lower_bound} derives a strictly positive \textbf{lower bound} on the batch variance,
i.e., $\operatorname{Tr}(\operatorname{Cov}(\mathbf{z}))>0$, providing a sufficient non-collapse certificate.
Together, these bounds ensure that $\boldsymbol{\Pi}_t$ provides a compact, regime-level description of the window that is both stable and expressive.

\paragraph{Predictor.}
With the latent dynamics discretized into a stable code sequence $\boldsymbol{\Pi}_t$, the remaining challenge is to forecast the system's evolution. 
Recognizing that anomalies often originate from conflicting temporal scales, we employ a dual-branch predictor to explicitly decouple these dynamics.
Both branches operate on the same history $\boldsymbol{\Pi}_t$ but target distinct future resolutions:

\begin{enumerate}[label=(\roman*), leftmargin=*]
    \item \textbf{Fine Predictor (Micro-Dynamics):}
    % [Motivation: Tracking details]
    To capture high-frequency volatility and local point anomalies, a Transformer maps $\boldsymbol{\Pi}_t$ to a sequence of fine-grained predictions
    $\hat{\boldsymbol{\Pi}}_{t+1}^{\text{fine}} \in \mathbb{R}^{P \times K}$. This branch preserves the patch-wise resolution, ensuring that transient disturbances are not smoothed out.

    \item \textbf{Coarse Predictor (Macro-Dynamics):}
    % [Motivation: Summarizing trends]
    To distill low-frequency trends and global distribution shifts, a learnable query token $\mathbf{q}$ aggregates the entire history into a single global prediction
    $\hat{\boldsymbol{\Pi}}_{t+1}^{\text{coarse}} \in \mathbb{R}^{1 \times K}$ via cross-attention:
    \begin{equation}
        \hat{\boldsymbol{\Pi}}_{t+1}^{\text{coarse}}
        = \operatorname{CrossAttn}\!\big(\mathbf{q}, \boldsymbol{\Pi}_t\big).
    \end{equation}
    This mechanism forces the model to abstract away local noise and focus on the trajectory of the system state.
\end{enumerate}
For supervision, the EMA branch generates the corresponding multi-scale targets $\boldsymbol{\Pi}^{\mathrm{fine}}_{t+1}$ and $\boldsymbol{\Pi}^{\mathrm{coarse}}_{t+1}$ using the shared codebook $\mathcal{Q}$.
\paragraph{Decoder.}

Because prediction alone can encourage overly abstract codes, an auxiliary decoder $\mathcal{D}$ reconstructs the input patches $\mathbf{X}^{\text{fine}}_t$ from the online soft-quantized embeddings $\mathbf{z}_t$ to retain signal-level semantics.
The reconstruction is evaluated after inverting RevIN using the cached statistics.

SC-JEPA achieves proactive anomaly prediction via multi-resolution forecasting of regime-level dynamics from the current window.
We next formalize the unified training objective in \Cref{sec:loss}.

\begin{table*}[htbp]
\centering
\caption{Full results on five benchmark datasets. Baselines are grouped by modeling characteristics. All metrics are reported in percentage (\%) as the mean over 5 runs with different random seeds.}
\label{main_exp}
\setlength{\tabcolsep}{2.5pt} 
\renewcommand{\arraystretch}{1.2}
\resizebox{\textwidth}{!}{%
\begin{tabular}{l|l|cccc|cccc|cccc|cccc|cccc}
\toprule
\textbf{Category} & \textbf{Models} 
& \multicolumn{4}{c|}{\textbf{MSL}} 
& \multicolumn{4}{c|}{\textbf{SMAP}} 
& \multicolumn{4}{c|}{\textbf{SWaT}} 
& \multicolumn{4}{c|}{\textbf{PSM}}
& \multicolumn{4}{c}{\textbf{SMD}} \\
\cline{3-22}
 &  & F1 & AUC & Prec & Rec
 & F1 & AUC & Prec & Rec
 & F1 & AUC & Prec & Rec
 & F1 & AUC & Prec & Rec 
 & F1 & AUC & Prec & Rec \\
\hline

\multirow{2}{*}{Traditional}
& K-Means~\cite{macqueen1967kmeans}
& 20.63 & 52.17 & 19.27 & 31.34
& 8.74  & 39.82 & 10.52 & 10.16
& 14.45 & 62.20 & 7.99 & \textbf{77.05}
& 42.75 & 51.37 & 30.40 & 78.89
& 14.79 & 57.49 & \textbf{35.83} & 20.64 \\

& DeepSVDD~\cite{deepsvdd}
& 23.46 & 53.67 & 14.78 & 73.87
& 22.05 & 44.07 & 12.90 & \textbf{81.99}
& 15.28 & 58.22 & 12.65 & 20.00
& 44.45 & 49.95 & 30.31 & 89.38
& 9.41 & 47.31 & 6.25 & 19.15 \\

\hline

\multirow{4}{*}{Generalist}
& LSTM-VAE~\cite{park2018multimodal}
& 22.42 & 55.22 & 20.84 & 32.43
& 22.77 & 52.97 & 14.35 & 55.39
& 54.60 & 79.91 & 64.92 & 50.07
& 56.73 & 71.95 & 49.32 & 66.86
& 10.04 & 56.18 & 5.78 & 38.29 \\

& iTransformer~\cite{itransformer}
& 27.25 & 64.61 & 16.71 & 75.40
& \cellcolor{gray!20}33.00 & 60.91 & 22.67 & 64.32
& 70.49 & 82.10 & \underline{98.17} & 55.00
& 54.12 & 63.09 & 37.37 & \underline{98.29}
& 14.20 & 57.94 & 8.11 & \underline{62.34} \\

& PatchTST~\cite{patchtst}
& 26.98 & 60.43 & \textbf{75.00} & 16.45
& 30.06 & 61.62 & 21.76 & 49.06
& 70.64 & 81.93 & 94.03 & 57.58
& \cellcolor{gray!20}58.17 & \cellcolor{gray!20}75.76 & \textbf{64.68} & 52.86
& 11.79 & 47.76 & 6.58 & 57.02 \\

& TS2Vec~\cite{ts2vec}
& 23.48 & \cellcolor{gray!20}64.86 & 31.45 & 18.82
& 32.81 & 61.48 & 23.81 & 52.81
& 67.00 & \cellcolor{gray!20}83.76 & 80.95 & 57.17
& 48.43 & 72.13 & 32.39 & 95.96
& 14.78 & 56.28 & 8.83 & 50.43 \\

\hline

\multirow{1}{*}{LLM}
& Qwen2.5-3B~\cite{qwen2_5}
& 21.24 & 55.66 & 17.82 & 26.67
& 0.34  & \cellcolor{gray!20}63.61 & 1.67  & 0.19
& 17.29 & 56.75 & 41.82 & 14.43
& 47.44 & 65.65 & 50.44 & 46.57
& 13.00 & \cellcolor{gray!20}60.66 & 8.29  & 31.06 \\

\hline

\multirow{3}{*}{Hierarchical}
& TimesNet~\cite{timesnet}
& \cellcolor{gray!20}28.44 & 59.70 & 30.31 & 36.59
& 27.78 & 57.01 & 17.98 & 62.81
& 67.20 & 83.24 & 90.87 & 54.42
& 52.64 & 56.12 & 35.76 & \textbf{99.71}
& 12.51 & 56.33 & 7.64 & 35.96 \\

& MICN~\cite{wang2023micn}
& 20.00 & 60.18 & 11.56 & 75.29
& 20.81 & 51.16 & 13.75 & 55.58
& 48.31 & 79.38 & 77.99 & 35.33
& 48.61 & 54.97 & 34.64 & 82.09
& 12.05 & 52.21 & 7.72 & 28.72 \\

& Pathformer~\cite{chen2024pathformer}
& 21.69 & 61.86 & 13.14 & 68.24
& 23.58 & 58.67 & 14.10 & \underline{72.12}
& 14.90 & 59.87 & 9.20 & 40.98
& 51.77 & 54.41 & 35.04 & \underline{99.10}
& 14.32 & 56.49 & 9.41 & 36.60 \\

\hline

\multirow{3}{*}{Anomaly Prediction}
& PAD~\cite{jhin2023precursor}
& 21.67 & 55.97 & 12.19 & \textbf{97.78}
& 26.92 & 59.79 & 18.15 & 55.00
& 69.75 & 83.71 & 81.54 & \underline{62.79}
& 57.82 & 73.84 & \underline{57.54} & 58.57
& 12.52 & 58.53 & 6.71 & \textbf{93.40} \\

& A2P~\cite{park2025fail}
& 23.10 & 59.43 & 13.92 & 74.12
& 21.68 & 60.37 & 15.88 & 38.85
& 70.18 & 81.92 & \textbf{99.10} & 54.33
& 50.71 & 57.01 & 37.16 & 80.00
& 15.21 & 58.87 & 12.23 & 22.13 \\

& FCM~\cite{zhao2024abnormalityforecasting}
& 24.57 & 61.09 & 14.25 & \underline{90.00}
& 30.60 & 61.85 & \textbf{25.25} & 41.41
& 69.09 & 81.66 & 94.25 & 54.83
& 52.34 & 57.25 & 37.00 & 89.85
& \cellcolor{gray!20}19.25 & 60.44 & 15.59 & 28.94 \\

\hline

\multirow{2}{*}{JEPAs}
& TS-JEPA~\cite{ts-jepa}
& 25.49 & 60.33 & 17.10 & 56.84
& 26.57 & 57.38 & 18.23 & 49.73
& \cellcolor{gray!20}71.95 & 80.33 & 84.32 & 62.76
& 53.32 & 66.10 & 36.95 & 96.29
& 7.49 & 49.53 & \underline{19.11} & 20.43 \\

& \cellcolor{scjepagreen}\textbf{SC-JEPA (Ours)}
& \cellcolor{scjepagreen}\textbf{33.58}
& \cellcolor{scjepagreen}\textbf{66.08}
& \cellcolor{scjepagreen}\underline{35.87}
& \cellcolor{scjepagreen}40.80
& \cellcolor{scjepagreen}\textbf{33.64}
& \cellcolor{scjepagreen}\textbf{65.41}
& \cellcolor{scjepagreen}\underline{24.24}
& \cellcolor{scjepagreen}56.02
& \cellcolor{scjepagreen}\textbf{72.89}
& \cellcolor{scjepagreen}\textbf{84.95}
& \cellcolor{scjepagreen}98.00
& \cellcolor{scjepagreen}58.05
& \cellcolor{scjepagreen}\textbf{61.61}
& \cellcolor{scjepagreen}\textbf{77.85}
& \cellcolor{scjepagreen}55.01
& \cellcolor{scjepagreen}72.00
& \cellcolor{scjepagreen}\textbf{20.82}
& \cellcolor{scjepagreen}\textbf{62.29}
& \cellcolor{scjepagreen}18.18
& \cellcolor{scjepagreen}31.49 \\
\bottomrule
\end{tabular}%
}
\end{table*}

\subsection{Optimization Objective}
\label{sec:loss}
We first introduce the notation used to define the three coupled objectives. Let $\mathbf{p}_{t,i}\in\mathbb{R}^{K}$ denote the online soft code distribution for patch $i$ at window $t$
(Eq.~\ref{eq:softassign}), satisfying $\mathbf{p}_{t,i}\ge \mathbf{0}$ and $\mathbf{1}^\top \mathbf{p}_{t,i}=1$.
Let $\mathbf{z}_{t,i}=\sum_{k=1}^{K}p_{t,i,k}\mathbf{c}_k$ be the corresponding soft-quantized latent.
The fine predictor outputs $\hat{\mathbf{p}}^{\mathrm{fine}}_{t+1,i}$ and latent $\hat{\mathbf{z}}_{t+1,i}$ for each patch,
while the coarse predictor outputs a single distribution $\hat{\mathbf{p}}^{\mathrm{coarse}}_{t+1}$.
All targets $\mathbf{p}^{\mathrm{fine}}_{t+1,i}, \mathbf{p}^{\mathrm{coarse}}_{t+1}$ and $\mathbf{z}_{t+1,i}$
are produced by the EMA (target) branch.

\paragraph{Overall objective.}
We optimize the online encoder, codebook, predictors, and decoder using three terms:
\begin{equation}
\mathcal{L}
=
\mathcal{L}_{\mathrm{pred}}
+
\mathcal{L}_{\mathrm{code}}
+
\lambda_r\,\mathcal{L}_{\mathrm{rec}}.
\label{eq:total_loss}
\end{equation}
With the overall objective in place, we detail each term next.
\paragraph{(A) Predictive Objective ($\mathcal{L}_{\mathrm{pred}}$).}
We supervise online predictions using EMA targets at both fine (patch) and coarse (window) resolutions, so that the representation is jointly shaped by local precursors and global trends.

\begin{equation}
\mathcal{L}_{\mathrm{pred}}
=
\underbrace{\lambda_f \left( \mathcal{L}_{\mathrm{KL}}^{\mathrm{fine}} + \gamma \mathcal{L}_{\mathrm{MSE}}^{\mathrm{fine}} \right)}_{\text{Local Structure}}
+
\underbrace{\lambda_c \mathcal{L}_{\mathrm{KL}}^{\mathrm{coarse}}}_{\text{Global Dynamics}}.
\label{eq:pred_loss}
\end{equation}
The fine-grained components align code distributions and stabilize latent features, while the coarse objective captures high-level semantics:
\begin{align}
\mathcal{L}_{\mathrm{KL}}^{\mathrm{fine}} &= \sum_{i} D_{\mathrm{KL}}(\mathbf{p}^{\mathrm{fine}}_{t+1,i} \parallel \hat{\mathbf{p}}^{\mathrm{fine}}_{t+1,i}), \\
\mathcal{L}_{\mathrm{MSE}}^{\mathrm{fine}} &= \sum_{i} \| \mathbf{z}_{t+1,i} - \hat{\mathbf{z}}_{t+1,i} \|_2^2, \\
\mathcal{L}_{\mathrm{KL}}^{\mathrm{coarse}} &= D_{\mathrm{KL}}(\mathbf{p}^{\mathrm{coarse}}_{t+1} \parallel \hat{\mathbf{p}}^{\mathrm{coarse}}_{t+1}).
\end{align}
Here, $D_{\mathrm{KL}}$ denotes the Kullback--Leibler divergence. $\lambda_f$ and $\lambda_c$ balance patch-level supervision with window-level consistency.
Since the targets $\mathbf{p}$ are probability vectors produced by the softmax codebook (i.e., $\mathbf{p}\in\Delta^{K-1}$), KL provides a principled objective for distributional alignment in the code space, rather than treating codes as unconstrained Euclidean regressands.
In addition, the momentum-updated EMA branch yields slowly varying targets, which stabilizes self-distillation and makes KL matching better conditioned during training.

\paragraph{(B) Codebook Objective ($\mathcal{L}_{\mathrm{code}}$).}
A well-behaved bottleneck requires more than defining soft assignments: during training, the code space must stay synchronized with the encoder features, and the assignment distributions must avoid becoming either overly diffuse or collapsed to a few indices. 
Accordingly, we regularize the soft codebook with (i) prototype--feature alignment and (ii) dual-entropy calibration:
\begin{equation}
\resizebox{\columnwidth}{!}{$
\mathcal{L}_{\mathrm{code}}
=
\underbrace{
\lambda_{\mathrm{emb}}\mathcal{L}_{\mathrm{emb}}
+
\lambda_{\mathrm{com}}\mathcal{L}_{\mathrm{com}}
}_{\text{Codebook alignment}}
+
\underbrace{
\lambda_{\mathrm{ent}}^{\mathrm{sample}} \mathcal{L}_{\mathrm{ent}}^{\mathrm{sample}}
-
\lambda_{\mathrm{ent}}^{\mathrm{batch}} \mathcal{L}_{\mathrm{ent}}^{\mathrm{batch}}
}_{\text{Entropy}}.
$}
\end{equation}
\label{loss code}

The alignment terms couple continuous encoder features $\mathbf{h}_{t,i}$ with their soft-quantized embeddings $\mathbf{z}_{t,i}$ via a bidirectional objective:
\begin{align}
\mathcal{L}_{\mathrm{emb}} &= \textstyle \sum_{i} \| \operatorname{sg}(\mathbf{z}_{t,i}) - \mathbf{h}_{t,i} \|_2^2, \\
\mathcal{L}_{\mathrm{com}} &= \textstyle \sum_{i} \| \mathbf{z}_{t,i} - \operatorname{sg}(\mathbf{h}_{t,i}) \|_2^2,
\end{align}
where $\operatorname{sg}(\cdot)$ denotes the stop-gradient operator.

Meanwhile, we calibrate the assignment distributions through dual-entropy control:
\begin{equation}
\mathcal{L}_{\mathrm{ent}}^{\mathrm{sample}} = \mathbb{E}[H(\mathbf{p})], \quad
\mathcal{L}_{\mathrm{ent}}^{\mathrm{batch}} = H(\mathbb{E}[\mathbf{p}]).
\end{equation}
In particular, minimizing $\mathcal{L}_{\mathrm{ent}}^{\mathrm{sample}}$ sharpens per-sample assignments and encourages decisive code selection, whereas maximizing $\mathcal{L}_{\mathrm{ent}}^{\mathrm{batch}}$ promotes diverse code usage across the dataset, mitigating index collapse.
Together, these terms keep the prototype set aligned with the evolving feature space while maintaining stable, informative code distributions for downstream prediction.

\paragraph{(C) Reconstruction Objective ($\mathcal{L}_{\mathrm{rec}}$).}
To anchor latent representations and prevent collapse, we minimize the reconstruction error between the input patches and their reconstructions (after denormalization):
\begin{equation}
\mathcal{L}_{\mathrm{rec}} = \sum_{i=1}^{P} \big\|\hat{\mathbf{X}}_{t,i} - \mathbf{X}_{t,i}\big\|_2^2.
\end{equation}
Here $\hat{\mathbf{X}}_{t,i}$ is denormalized using the cached RevIN statistics of window $t$ before computing $\mathcal{L}_{\mathrm{rec}}$.

\section{Experiments}

\subsection{Experimental Setup}
\label{subsec:setup}

\paragraph{Anomaly Prediction Setting.}
Following prior anomaly prediction frameworks~\cite{jhin2023precursor,park2025fail,zhao2024abnormalityforecasting}, we partition each time series into non-overlapping windows of length 100.
For each sample, the current window $\mathbf{X}_t$ is used as the model input, while the next window $\mathbf{X}_{t+1}$ is used to define the supervision label $Y_{t+1}$.
If $\mathbf{X}_{t+1}$ contains at least one anomalous time point, we assign $Y_{t+1}=1$; otherwise, we assign $Y_{t+1}=0$.
The model is then trained to use only $\mathbf{X}_t$ to predict $Y_{t+1}$, namely whether the next non-overlapping window will contain an anomaly.
Under this formulation, $\mathbf{X}_t$ serves as the precursor context, and the task evaluates whether the model can identify early warning signals before anomalies fully emerge.
We report window-level Precision, Recall, F1, and AUC.

\paragraph{Datasets and Baselines.}

We evaluate on five widely-used multivariate anomaly detection benchmarks spanning spacecraft telemetry, industrial control, and cloud services.
Mars Science Laboratory (MSL) and Soil Moisture Active Passive (SMAP) are NASA spacecraft telemetry datasets~\cite{hundman2018telemanom}.
Secure Water Treatment (SWaT) is an industrial control system (ICS) water-treatment testbed with attack scenarios~\cite{goh2017swat}.
Pooled Server Metrics (PSM) is a public server monitoring dataset collected from eBay production machines~\cite{su2019omnianomaly}.
Server Machine Dataset (SMD) is a large-scale multivariate time-series benchmark collected from multiple server machines~\cite{su2019omnianomaly}.
To evaluate SC-JEPA, we benchmark against six categories of baselines:
(1) traditional methods: K-Means~\cite{macqueen1967kmeans}, DeepSVDD~\cite{deepsvdd};
(2) generalist sequence models: LSTM-VAE~\cite{park2018multimodal}, iTransformer~\cite{itransformer}, PatchTST~\cite{patchtst}, TS2Vec~\cite{ts2vec};
(3) LLM: Qwen2.5-3B~\cite{qwen2_5};
(4) hierarchical models: TimesNet~\cite{timesnet}, MICN~\cite{wang2023micn}, Pathformer~\cite{chen2024pathformer};
(5) anomaly prediction methods: PAD~\cite{jhin2023precursor}, A2P~\cite{park2025fail}, FCM~\cite{zhao2024abnormalityforecasting};
and (6) JEPA-based models: TS-JEPA~\cite{ts-jepa}.

\textbf{Training \& Evaluation Protocol.}
We convert each multivariate stream into non-overlapping context–target pairs with stride 100, yielding windows with context length $T_c=100$ and target length $T_t=100$. Each window is tokenized into $P=5$ non-overlapping patches of length $L=20$. We use a two-phase pipeline. In self-supervised pre-training, we train on the official training set without labels and use a 9:1 train--validation split for model selection and early stopping.
For downstream training on the official test set, we adopt a 6:2:2 chronological split: we train the classifier on the training split, select the detection threshold on the validation split, and then fix it for final evaluation on the held-out test split.
We report window-level Precision, Recall, F1, and AUC.
Full implementation details, hyperparameters, and the downstream protocol are provided in Appendix~\ref{app:implementation} and Appendix~\ref{app:downstream_protocol}.

\subsection{Results}
\label{subsec:main_results}

\paragraph{Main results.}
Table~\ref{main_exp} reports the mean anomaly prediction results on five benchmarks using the same data splits and five random seeds. Full results are available in Appendix~\ref{app:extended_results}. Under this setting, SC-JEPA achieves the top AUC on all datasets, suggesting consistently strong ranking quality for upcoming abnormal windows under our precursor-based definition. Moreover, SC-JEPA attains best F1 performance across benchmarks, indicating that its ranking quality can translate into effective decisions under the shared threshold-selection procedure.

The trends of competing baselines are broadly consistent with their modeling assumptions. Classical methods such as K-Means and DeepSVDD rely on static similarity or compactness in a feature space, which can be less sensitive when precursor windows deviate only mildly from normal patterns. Reconstruction-based models such as LSTM-VAE may also be less aligned with a precursor setting, since reconstruction error often becomes more distinctive after abnormal behavior has developed. In addition, a set of advanced time-series backbones, including TimesNet, TS2Vec, PatchTST, and iTransformer, provide strong sequence representations and capture long-range dependencies. In our anomaly prediction setting, however, these general-purpose backbones are not directly optimized for highlighting weak early-stage transitions toward future abnormalities, so their scores can be less separable when the precursor signal is subtle and noisy. Likewise, the inferior performance of TS-JEPA highlights the limitations of a pure continuous predictive model; lacking both a discrete codebook and multi-resolution modeling, it suffers from optimization instability and fails to capture diverse precursor dynamics. Finally, while PAD is explicitly designed for anomaly prediction, it relies on a single-scale objective that restricts its ability to decouple complex latent evolution from stochastic noise.

\begin{figure*}[t]
    \centering
    \begin{subfigure}[t]{0.23\textwidth}
        \centering
        \includegraphics[width=\linewidth]{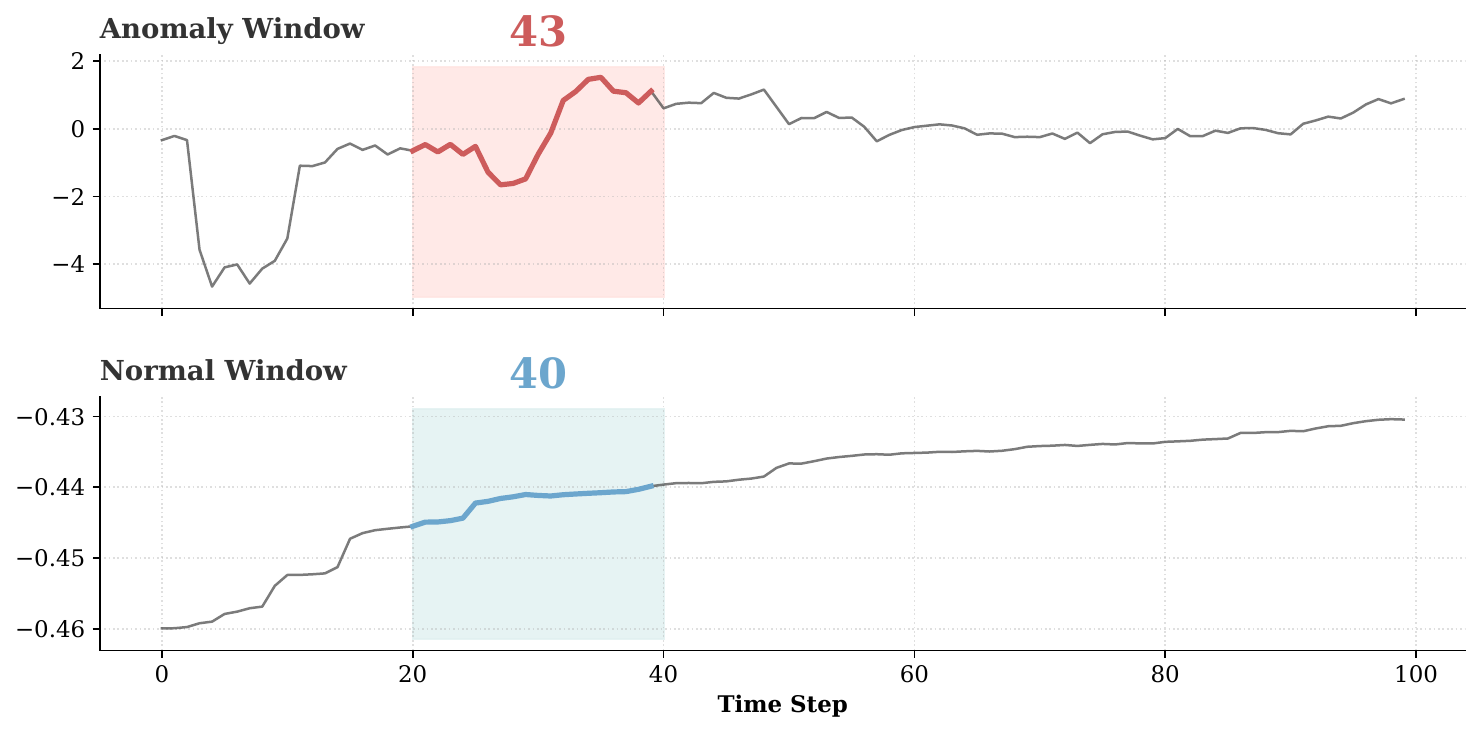}
        \caption*{\scriptsize Anomalous 43; Normal 40}
    \end{subfigure}
    \hfill
    \begin{subfigure}[t]{0.23\textwidth}
        \centering
        \includegraphics[width=\linewidth]{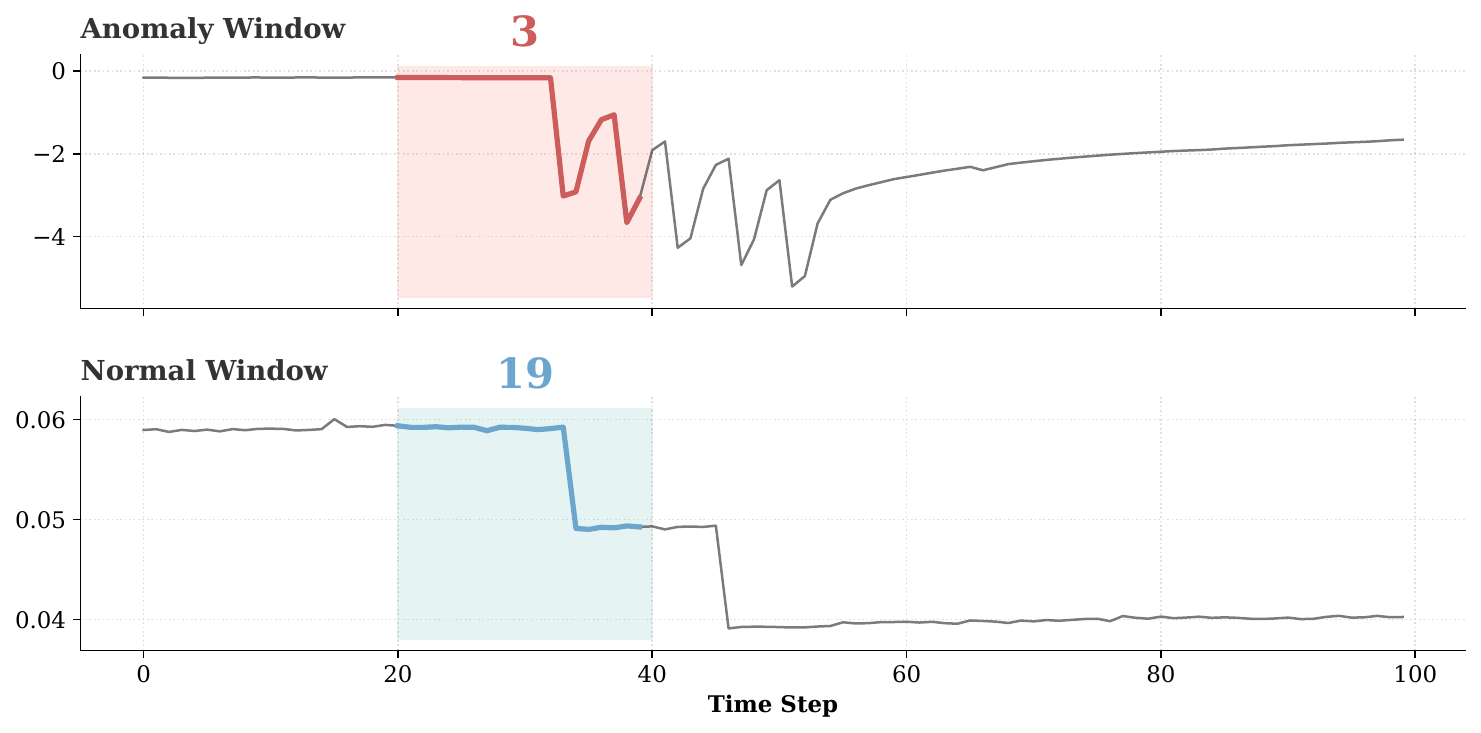}
        \caption*{\scriptsize Anomalous 3; Normal 19}
    \end{subfigure}
    \hfill
    \begin{subfigure}[t]{0.23\textwidth}
        \centering
        \includegraphics[width=\linewidth]{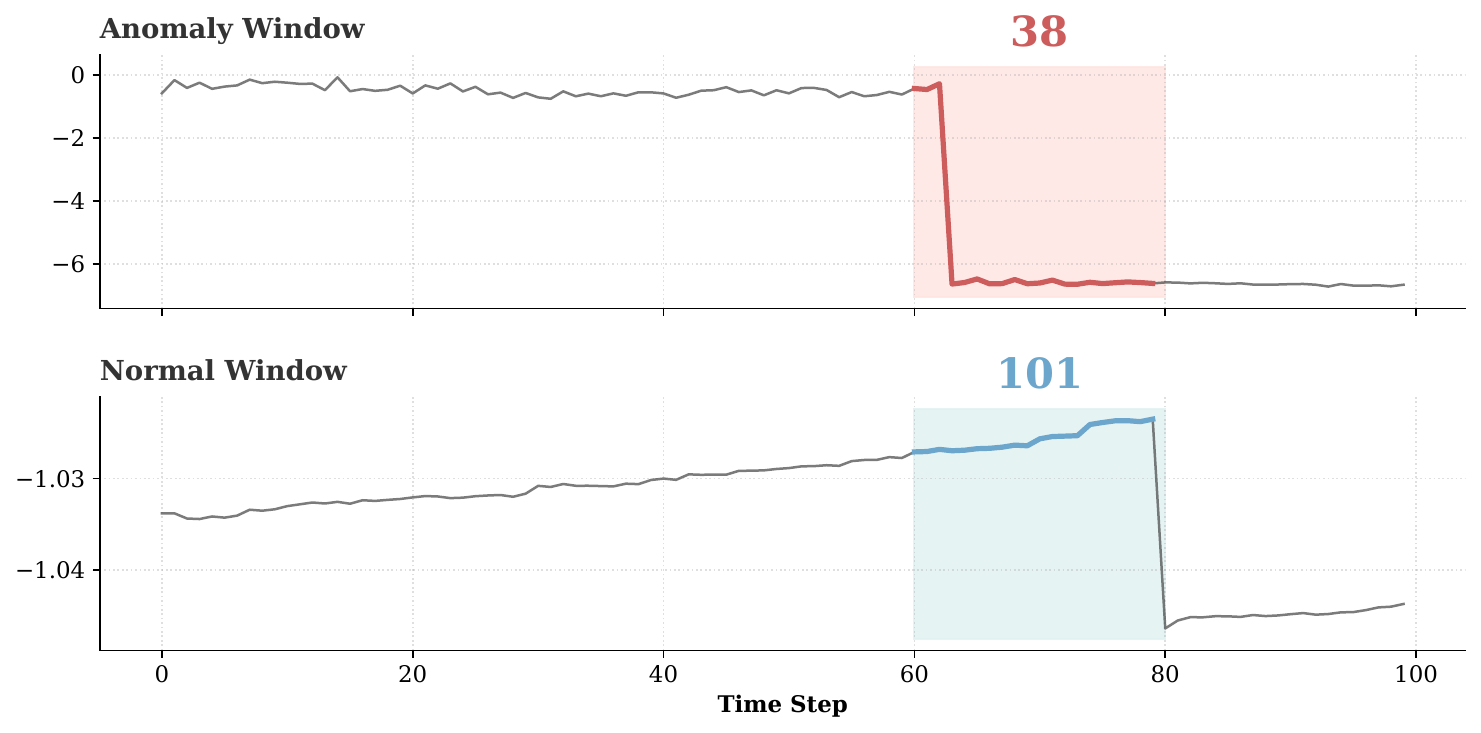}
        \caption*{\scriptsize Anomalous 38; Normal 101}
    \end{subfigure}
    \hfill
    \begin{subfigure}[t]{0.23\textwidth}
        \centering
        \includegraphics[width=\linewidth]{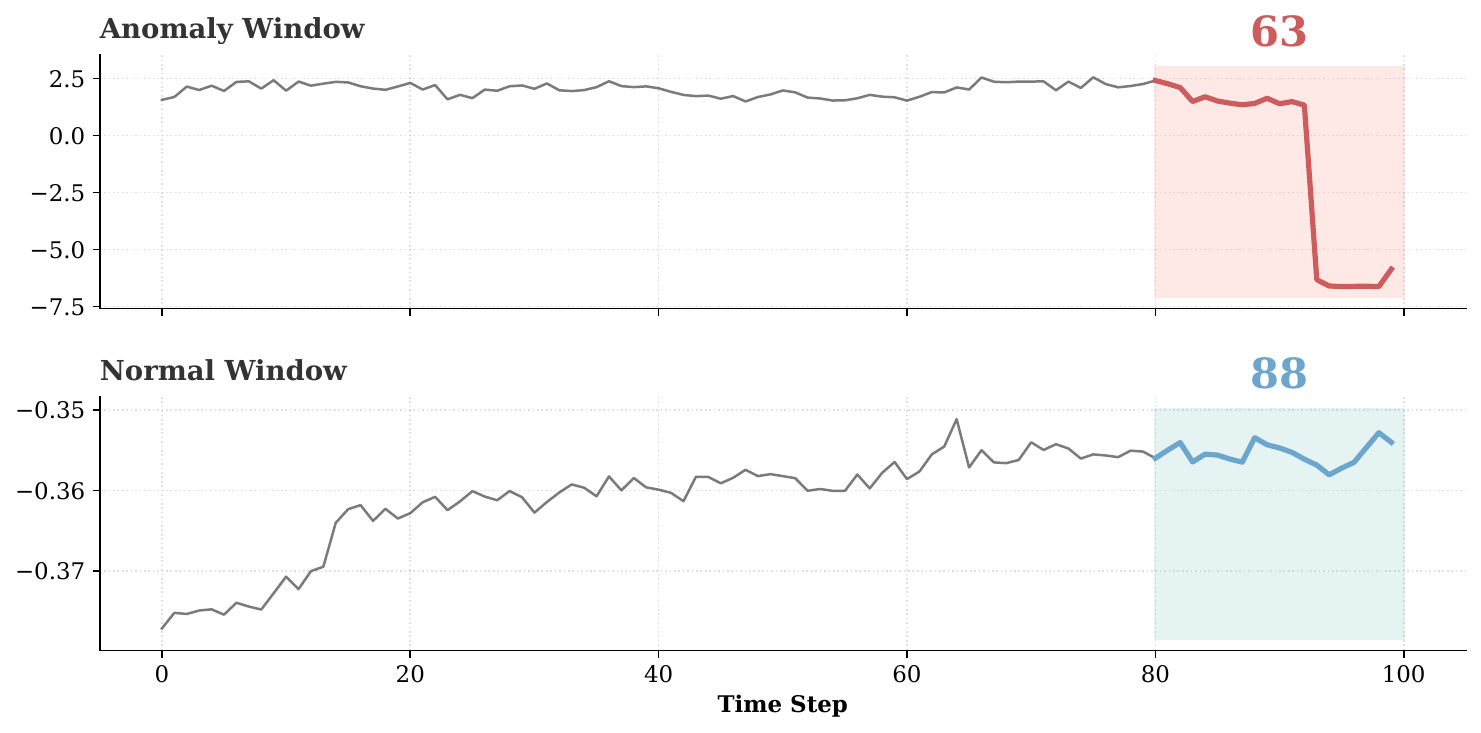}
        \caption*{\scriptsize Anomalous 63; Normal 88}
    \end{subfigure}

    \caption{Window-level comparisons of dominant latent codes between anomalous and normal windows on the PSM dataset. Shaded regions indicate patch-level segments where the dominant code is selected.}
    \label{analysis:code_top}
\end{figure*}

\begin{table}[t]
\centering
\caption{In-domain and cross-domain generalization results. Mean F1 and AUC (\%) are reported over 5 runs. ID: in-domain; CD: cross-domain; $\Delta$: relative change from ID to CD.}
\label{exp:generalization}

\setlength{\tabcolsep}{2.5pt}
\renewcommand{\arraystretch}{1.05}

\resizebox{\columnwidth}{!}{%
\begin{tabular}{l|l|cc|cc|cc|cc}
\toprule
\multirow{2}{*}{\textbf{Models}} & \multirow{2}{*}{\textbf{Settings}} 
& \multicolumn{2}{c|}{\textbf{MSL}} 
& \multicolumn{2}{c|}{\textbf{SMAP}} 
& \multicolumn{2}{c|}{\textbf{SWaT}} 
& \multicolumn{2}{c}{\textbf{PSM}} \\
\cline{3-10}
 &  & F1 & AUC & F1 & AUC & F1 & AUC & F1 & AUC \\
\hline

\multirow{3}{*}{\textbf{PatchTST}} 
& ID    
& 26.98 & 60.43 & 30.06 & 61.62 & 70.64 & 81.93 & 58.17 & 75.76 \\
& CD    
& \cellcolor{gray!20}27.08 & 60.80 
& 26.21 & \cellcolor{gray!20}58.77 
& 54.34 & 68.76 
& \cellcolor{gray!20}50.07 & 53.72 \\
& $\Delta$
& {\color{blue!60!black}+0.37\%} & {\color{blue!60!black}+0.61\%}
& {\color{blue!60!black}-12.81\%} & {\color{blue!60!black}-4.63\%}
& {\color{blue!60!black}-23.07\%} & {\color{blue!60!black}-16.07\%}
& {\color{blue!60!black}-13.92\%} & {\color{blue!60!black}-29.09\%} \\
\hline

\multirow{3}{*}{\textbf{TS2Vec}} 
& ID    
& 23.48 & 64.86 & 32.81 & 61.48 & 67.00 & 83.76 & 48.43 & 72.13 \\
& CD    
& 22.87 & \cellcolor{gray!20}62.00 
& \cellcolor{gray!20}26.99 & 57.54 
& \cellcolor{gray!20}65.08 & \cellcolor{gray!20}82.91 
& 48.15 & \cellcolor{gray!20}56.22 \\
& $\Delta$
& {\color{blue!60!black}-2.60\%} & {\color{blue!60!black}-4.41\%}
& {\color{blue!60!black}-17.74\%} & {\color{blue!60!black}-6.41\%}
& {\color{blue!60!black}-2.87\%} & {\color{blue!60!black}-1.01\%}
& {\color{blue!60!black}-0.58\%} & {\color{blue!60!black}-22.06\%} \\
\hline

\multirow{3}{*}{\textbf{SC-JEPA}} 
& ID    
& 33.58 & 66.08 & 33.64 & 65.41 & 72.89 & 84.95 & 61.61 & 77.85 \\
& CD    
& \textbf{33.48} & \textbf{67.13} & \textbf{33.81} & \textbf{68.88} & \textbf{71.10} & \textbf{82.96} & \textbf{56.87} & \textbf{66.80} \\
& $\Delta$
& {\color{blue!60!black}-0.30\%} & {\color{blue!60!black}+1.59\%}
& {\color{blue!60!black}+0.51\%} & {\color{blue!60!black}+5.30\%}
& {\color{blue!60!black}-2.46\%} & {\color{blue!60!black}-2.34\%}
& {\color{blue!60!black}-7.69\%} & {\color{blue!60!black}-14.19\%} \\
\bottomrule
\end{tabular}%
}

\vspace{-4mm}
\end{table}

\subsection{Generality}
\label{subsec:generalization}
A practical time-series model should remain effective under distribution shifts, where the pre-training data may differ substantially from the deployment environment. To assess such transferability, Table~\ref{exp:generalization} reports generality results under two pre-training settings. In-domain pre-training uses the official training set of the target dataset, while cross-domain pre-training uses the union of the other three datasets, with the target dataset fully excluded during pre-training. In both cases, we keep the downstream evaluation protocol identical on the target dataset, so the only difference comes from the source of pre-training data. We select TS2Vec and PatchTST for this generality study since they perform strongly in our main table under our setting, and thus provide meaningful SOTA references for evaluating transferability.

With this setup, Table~\ref{exp:generalization} shows that cross-domain transfer is generally more challenging, and performance can decrease when the pre-training data comes from different datasets, which is consistent with the presence of domain shift. Notably, SC-JEPA continues to perform competitively across benchmarks. We observe that its AUC remains comparable to the in-domain setting on several datasets and can improve modestly in some cases, suggesting that the learned scores preserve useful ranking information even when the target dataset is excluded during pre-training. On more challenging data such as PSM, all methods experience a clear drop, yet SC-JEPA retains relatively strong performance compared to the evaluated backbones. Overall, these results provide additional evidence that the model transfers reasonably well across datasets, particularly when assessed through threshold-free ranking metrics.

\begin{table}[t]
\centering
\caption{Ablation results on five benchmark datasets (\%).}
\label{exp:ablation}

\setlength{\tabcolsep}{2pt}
\renewcommand{\arraystretch}{1.15}
\resizebox{\columnwidth}{!}{%
\begin{tabular}{l|cc|cc|cc|cc|cc}
\noalign{\hrule height 1.2pt}
\textbf{Model Variant}
& \multicolumn{2}{c|}{\textbf{MSL}}
& \multicolumn{2}{c|}{\textbf{SMAP}}
& \multicolumn{2}{c|}{\textbf{SWaT}}
& \multicolumn{2}{c|}{\textbf{PSM}}
& \multicolumn{2}{c}{\textbf{SMD}} \\
\hline
\textbf{Metric}
& F1 & AUC & F1 & AUC & F1 & AUC & F1 & AUC & F1 & AUC \\
\hline

w/o KL Divergence       
& 29.01 & 53.90 & 28.71 & 60.19 & 67.33 & 79.95 & 54.93 & 70.83 
& 12.63 & 53.61 \\

w/o Reconstruction   
& 25.96 & 52.68 & 25.15 & 52.80 & 14.77 & 53.30 & 53.03 & 69.84 
& 10.18 & 50.44 \\

w/o Predictive Objective          
& 28.88 & 53.17 & 27.05 & 59.81 & 59.47 & 78.60 & 58.89 & 70.17 
& 10.04 & 50.24 \\
\rowcolor{gray!15}
w/o Codebook Loss     
& 31.62 & 58.93 & 31.53 & 62.06 & 72.64 & 82.63 & 60.83 & 75.85 
& 12.82 & 57.36 \\
\rowcolor{gray!30}
w/o Codebook Module  
& 21.82 & 43.02 & 21.69 & 51.00 & 11.51 & 50.00 & 53.03 & 46.61 
& 12.45 & 50.00 \\

w/o Temporal Downsampling          
& 28.77 & 63.16 & 30.39 & 61.34 & 71.62 & 83.60 & 60.79 & 71.72 
& 15.50 & 60.93 \\

\hline
\rowcolor{scjepagreen}
\textbf{Full Model}
& \textbf{33.58} & \textbf{66.08}
& \textbf{33.64} & \textbf{65.41}
& \textbf{72.89} & \textbf{84.95}
& \textbf{61.61} & \textbf{77.85}
& \textbf{20.82} & \textbf{62.29} \\

\noalign{\hrule height 1.2pt}
\end{tabular}}
\vspace{-4mm}
\end{table}

\subsection{Model Analysis}

\subsubsection{Ablation Study}
Since SC-JEPA integrates multiple mechanisms, it is important to attribute the observed gains to the intended design choices rather than a single component. We therefore conduct a systematic ablation study by selectively removing modules, as summarized in Table~\ref{exp:ablation}. We observe that the predictive objective is the primary engine for capturing temporal dependencies; removing it causes a significant drop, and the KL term proves superior to MSE alone by enforcing precise distributional alignment. Furthermore, the results validate that multi-resolution modeling is essential for outperforming baselines, as it enables the capture of coarse-grained precursors that single-scale models miss. Conversely, reconstruction serves as a non-negotiable anchor; its removal triggers model collapse, confirming its role in preventing latent degeneration.

Among all variants, the most consequential one is removing the soft codebook bottleneck (w/o Codebook Module):
performance collapses to near-random, while removing only the auxiliary regularizers yields a much smaller drop. It suggests that beyond providing discrete, prototype-anchored regime codes, the bottleneck may also play a critical role in optimization stability.
This naturally raises a question: what stabilizing effect does the bottleneck introduce that the auxiliary losses alone do not?

The ablations offer an empirical explanation consistent with our formulation.
With the codebook in place, self-distillation operates as KL matching between probability codes on the simplex (Eq.~\ref{eq:softassign}),
so the predictor tracks a bounded distributional trajectory instead of performing unconstrained regression in $\mathbb{R}^D$.
In conjunction with slowly varying EMA targets, this bounded code space appears to provide a more stable reference for self-distillation~\cite{zhou2022ibot}.
Once we remove the bottleneck, this constrained code space disappears, so KL alignment loses its natural domain,
and training degenerates into an ill-conditioned matching problem, consistent with the observed near-collapse.
By contrast, removing only the auxiliary codebook losses causes only a mild drop for w/o Codebook Loss.
These auxiliary losses mainly serve as an outperformance mechanism that further boosts an already strong model.

%---------------
%var check
\begin{table}[t]
\centering
\caption{Variance ratio comparison on three datasets. Lower values indicate less representation collapse; gray text shows the relative decrease versus the variant without codebook.}
\label{tab:variance_ratio}
\renewcommand{\arraystretch}{1.12}

\resizebox{\linewidth}{!}{

\begin{tabular}{@{} l l l@{\hspace{2mm}}r l@{\hspace{2mm}}r l@{\hspace{2mm}}r @{}}
\toprule
\textbf{Dataset} & \textbf{Method} & \multicolumn{2}{l}{\textbf{Top-1 $\downarrow$}} & \multicolumn{2}{l}{\textbf{Top-5 $\downarrow$}} & \multicolumn{2}{l}{\textbf{Top-10 $\downarrow$}} \\
\midrule
PSM 
& without codebook 
& 0.5334 & 
& 1.0000 & 
& 1.0000 & \\
& with codebook (ours) 
& 0.1765 & {\scriptsize\color{gray}($\downarrow$66.9\%)} 
& 0.4347 & {\scriptsize\color{gray}($\downarrow$56.5\%)} 
& 0.6778 & {\scriptsize\color{gray}($\downarrow$32.2\%)} \\
\cmidrule(lr){1-8}

MSL 
& without codebook 
& 0.3292 & 
& 0.9088 & 
& 0.9154 & \\
& with codebook (ours) 
& 0.1804 & {\scriptsize\color{gray}($\downarrow$45.2\%)} 
& 0.5194 & {\scriptsize\color{gray}($\downarrow$42.8\%)} 
& 0.6607 & {\scriptsize\color{gray}($\downarrow$27.8\%)} \\
\cmidrule(lr){1-8}

SWaT 
& without codebook 
& 0.3251 & 
& 0.9271 & 
& 0.9344 & \\
& with codebook (ours) 
& 0.1143 & {\scriptsize\color{gray}($\downarrow$64.8\%)} 
& 0.3190 & {\scriptsize\color{gray}($\downarrow$65.6\%)} 
& 0.4041 & {\scriptsize\color{gray}($\downarrow$56.8\%)} \\
\bottomrule
\end{tabular}
}
\vspace{-4mm}
\end{table}

\subsubsection{Visualization of Learned Regimes}
\label{sec:vis_analysis}

In the Introduction, we hypothesized that a discrete bottleneck would sharpen the distinction between normal dynamics and structural anomalies. Furthermore, in Section~\ref{archi}, we designed the codebook $\mathcal{Q}$ specifically to capture regime-like patterns via finite prototypes $\{\mathbf{c}_k\}$.
Here, we empirically verify whether the learned representations align with these methodological motivations.

First, we examine the activation statistics of the codebook. 
Figure~\ref{analysis:code-his} reports a quantitative study on the PSM dataset, ranking codes by the absolute frequency gap between normal and anomalous windows. We observe that a small subset of codes exhibits pronounced activation gaps, suggesting that the model naturally learns to partition the latent space into condition-dependent regions without explicit supervision.

To further validate this structure qualitatively, Figure~\ref{analysis:code_top} maps the most discriminative codes back to their original time-series patches. 
We observe that, within the same temporal context, distinct codes consistently correspond to distinct physical behaviors. This alignment confirms that the discrete formulation effectively disentangles diverse local dynamics into a finite set of identifiable regimes. Consequently, the codebook successfully distinguishes precursor-like patterns from normal dynamics, providing a discriminative structural basis for anomaly prediction even under distribution shift.
\begin{figure}[t]  
    \centering
    \includegraphics[width=\linewidth]{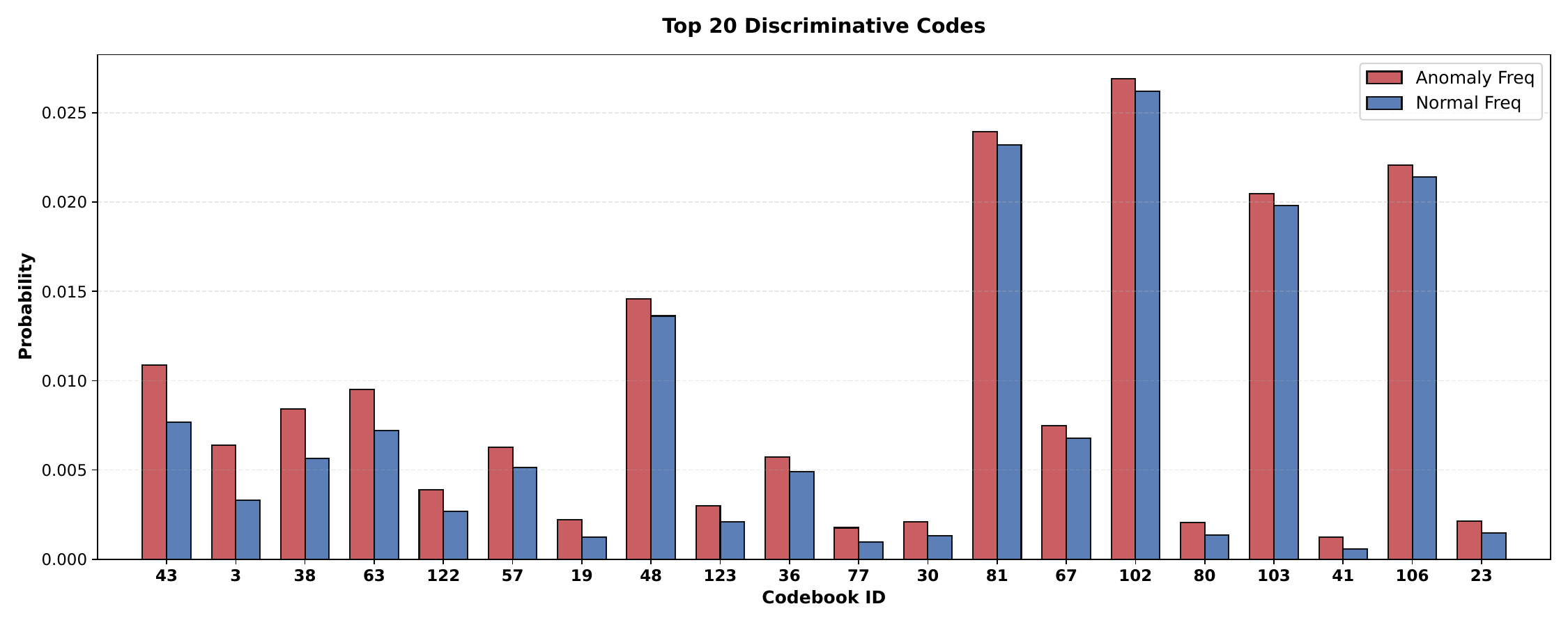}
    \caption{Histogram comparing the occurrence probability of latent codes in anomalous versus normal windows on the PSM test set. The codes are sorted by their anomaly–normal probability difference (from largest to smallest). PSM is used because it provides a large test set with reliable window-level anomaly labels.}
    \label{analysis:code-his} 
    \vspace{-5mm}
\end{figure}

\section{Conclusion}
This work introduces SC-JEPA, advancing anomaly prediction from reactive detection to proactive world modeling. By synergizing multi-resolution views with a soft-quantized bottleneck, our architecture effectively captures scale-variant precursors while mitigating the representation collapse often observed in continuous self-supervision. Our empirical and theoretical analysis reveals that the discrete bottleneck serves as a critical regularizer, enabling stable optimization without explicit negative sampling. Extensive experiments confirm that SC-JEPA achieves state-of-the-art performance in early-warning scenarios, suggesting that discrete latent dynamics offer a robust foundation for future research in reliable autonomous systems.

%\section*{Acknowledgements}
\section*{Impact Statement}
This work advances machine learning methods for early warning from multivariate time series. By improving the stability of latent predictive learning and modeling precursor patterns across temporal scales, SC-JEPA may support earlier identification of emerging failures in industrial, infrastructure, and computing systems. Such capabilities could help reduce downtime, operational cost, and safety risks when deployed with appropriate human oversight. However, inaccurate predictions may lead to missed failures or unnecessary interventions, and performance may degrade under distribution shifts or in domains not represented by the training data. Therefore, the method should be carefully validated for each application and used as decision support rather than as a fully autonomous safety mechanism.

\bibliography{example_paper}
\bibliographystyle{icml2026}

%%%%%%%%%%%%%%%%%%%%%%%%%%%%%%%%%%%%%%%%%%%%%%%%%%%%%%%%%%%%%%%%%%%%%%%%%%%%%%%
%%%%%%%%%%%%%%%%%%%%%%%%%%%%%%%%%%%%%%%%%%%%%%%%%%%%%%%%%%%%%%%%%%%%%%%%%%%%%%%
% APPENDIX
%%%%%%%%%%%%%%%%%%%%%%%%%%%%%%%%%%%%%%%%%%%%%%%%%%%%%%%%%%%%%%%%%%%%%%%%%%%%%%%
%%%%%%%%%%%%%%%%%%%%%%%%%%%%%%%%%%%%%%%%%%%%%%%%%%%%%%%%%%%%%%%%%%%%%%%%%%%%%%%
\newpage
\appendix
\onecolumn

\appendix

\section*{Appendix Contents}

\noindent
\hyperref[app:methodology]{\textbf{A. Methodological Details}}
\dotfill
\pageref{app:methodology}

\par\smallskip
\hspace*{1.5em}
\hyperref[app:notations]{A.1 Notations}
\dotfill
\pageref{app:notations}

\par
\hspace*{1.5em}
\hyperref[app:downavg]{A.2 Time-Axis Downsampling}
\dotfill
\pageref{app:downavg}

\par
\hspace*{1.5em}
\hyperref[app:theory]{A.3 Theoretical Analysis of Latent Stability}
\dotfill
\pageref{app:theory}

\par
\hspace*{3em}
\hyperref[upper-bound]{A.3.1 Stability Upper Bound}
\dotfill
\pageref{upper-bound}

\par
\hspace*{3em}
\hyperref[app:lower_bound]{A.3.2 Non-Collapse Lower Bound}
\dotfill
\pageref{app:lower_bound}

\par\medskip
\noindent
\hyperref[app:reproducibility]{\textbf{B. Reproducibility Details}}
\dotfill
\pageref{app:reproducibility}

\par
\hspace*{1.5em}
\hyperref[app:implementation]{B.1 Implementation Details}
\dotfill
\pageref{app:implementation}

\par
\hspace*{1.5em}
\hyperref[app:pretrain]{B.2 SC-JEPA Pre-training Algorithm}
\dotfill
\pageref{app:pretrain}

\par
\hspace*{1.5em}
\hyperref[app:downstream_protocol]{B.3 Downstream Anomaly Prediction Protocol}
\dotfill
\pageref{app:downstream_protocol}

\par\medskip
\noindent
\hyperref[app:extended_results]{\textbf{C. Experimental Results}}
\dotfill
\pageref{app:extended_results}

\par\smallskip
\hspace*{1.5em}
\hyperref[app:main_table]{C.1 Main Benchmark}
\dotfill
\pageref{app:main_table}

\par
\hspace*{1.5em}
\hyperref[app:generality]{C.2 Generality}
\dotfill
\pageref{app:generality}

\par
\hspace*{1.5em}
\hyperref[app:ablation]{C.3 Ablation Studies}
\dotfill
\pageref{app:ablation}

\par\medskip
\noindent
\hyperref[app:viz]{\textbf{D. Visualizations and Qualitative Analysis}}
\dotfill
\pageref{app:viz}

\par\smallskip
\hspace*{1.5em}
\hyperref[app:framework]{D.1 Framework}
\dotfill
\pageref{app:framework}

\par
\hspace*{1.5em}
\hyperref[app:inference]{D.2 Inference Efficiency}
\dotfill
\pageref{app:inference}

\clearpage

% =========================================================
% Section B: Methodology & Algorithms
% =========================================================
\section{Methodological Details}
\label{app:methodology}

\subsection{Notations}
\label{app:notations}
\begin{table}[H]
\centering
\caption{Categorized notations.}
\label{tab:notation_refined}
\renewcommand{\arraystretch}{1.2}
\resizebox{\textwidth}{!}{%
\begin{tabular}{cll|cll}
\toprule
\textbf{Category} & \textbf{Symbol} & \textbf{Description} & \textbf{Category} & \textbf{Symbol} & \textbf{Description} \\
\midrule
\multirow{8}{*}{Dimensions} 
 & $T, V$ & Time series size ($T{\times}V$) & \multirow{4}{*}{Data Inputs} & $\mathbf{x}$ & Input time series \\
 & $T_w$ & Window length & & $\hat{\mathbf{W}}_t$ & Normalized window \\
 & $N$ & Number of windows & & $\mathbf{X}^{\text{fine}}_t$ & Fine view (patches) \\
 & $P$ & Num. patches ($T_w = P{\cdot}L$) & & $\mathbf{X}^{\text{coarse}}_t$ & Coarse view (downsampled) \\
\cline{4-6}
 & $L$ & Patch length & \multirow{7}{*}{Representations} & $\mathbf{h}_t$ & Encoder latent rep. \\
 & $K$ & Codebook size & & $\mathbf{c}_k$ & Learnable prototype \\
 & $D$ & Embedding dimension & & $\mathbf{p}_{t,i}$ & Code-assignment prob. \\
\cline{1-3}
\multirow{3}{*}{Hyperparams}
 & $\tau$ & Soft-max temperature & & $\boldsymbol{\Pi}_t$ & Code sequence \\
 & $\lambda_{\dots}$ & Loss balancing weights & & $\mathbf{z}_{t,i}$ & Soft-quantized embedding \\
 & $\gamma$ & MSE local weight & & $\mathbf{q}$ & Coarse query token \\
\cline{1-3}
\multirow{5}{*}{Architecture}
 & $\mathcal{E}_\theta$ & Online encoder & & & \\
\cline{4-6}
 & $\mathcal{E}_\xi$ & EMA encoder & \multirow{4}{*}{Outputs} & $\hat{\boldsymbol{\Pi}}^{\text{fine}}$ & Predicted fine code dist. \\
 & $\mathcal{Q}$ & Codebook bottleneck & & $\hat{\boldsymbol{\Pi}}^{\text{coarse}}$ & Predicted coarse code dist. \\
 & $\mathcal{D}$ & Reconstruction decoder & & $\hat{\mathbf{z}}_{t,i}$ & Predicted embedding \\
 & $\operatorname{sg}$ & Stop-gradient operator & & $\hat{\mathbf{X}}_{t,i}$ & Reconstructed patch \\
\cline{4-6}
 & & & \multirow{2}{*}{Objectives} & $\mathcal{L}_{\mathrm{pred/code}}$ & Predictive / Code losses \\
 & & & & $\mathcal{L}_{\mathrm{rec/total}}$ & Recon. / Total losses \\
\bottomrule
\end{tabular}%
}

\end{table}
\subsection{Time-Axis Downsampling}
\label{app:downavg}

\noindent\textbf{DownAvg operator.}
Given a normalized window $\hat{\mathbf{W}}_t \in \mathbb{R}^{T_w\times V}$ with $T_w=P\cdot L$, we define $\operatorname{DownAvg}_P(\cdot)$ as a time-axis downsampling operator that averages every $P$ consecutive time points (stride $P$) for each variable:
\begin{equation}
\mathbf{X}^{\text{coarse}}_t=\operatorname{DownAvg}_P(\hat{\mathbf{W}}_t)\in\mathbb{R}^{1\times L\times V}.
\end{equation}
Equivalently, for within-patch index $\ell\in[L]$ and variable $v\in[V]$:
\begin{equation}
\big[\mathbf{X}^{\text{coarse}}_t\big]_{1,\ell,v}
=\frac{1}{P}\sum_{j=1}^{P}\hat{\mathbf{W}}_t[(\ell-1)P+j,\;v].
\label{eq:downavg_appendix}
\end{equation}

\subsection{Theoretical Analysis of Latent Stability}
\label{app:theory}

In this subsection, we provide the formal proofs for the stability properties of the SC-JEPA latent space. We demonstrate how the soft codebook acts as a geometric constraint that explicitly bounds representation drift.
We provide two complementary guarantees: (i) an upper bound on representation drift (stability),
and (ii) a lower bound that rules out the trivial constant solution (non-collapse).
\subsubsection{Stability Upper Bound}
\label{upper-bound}
This part proves an explicit upper bound on representation drift in embedding space.
The argument follows a short chain: we first fix a bounded latent geometry via the soft codebook,
then establish a probability-to-geometry bridge ($\ell_1 \to \ell_2$), and finally combine this bridge with EMA smoothness and predictive alignment to obtain Theorem~\ref{thm:stability_app}.

\paragraph{Preliminaries.}

Continuous forecasting models can exhibit collapse or explosion, making the latent geometry unstable.
The soft codebook eliminates this degree of freedom by restricting every soft-quantized embedding to the convex hull of a finite prototype set.
In particular, the radius constant $M$ defined below will be the only scale factor that propagates into the final drift bound.

\begin{definition}[Soft Code Map]
\label{def:softmap_app}
Let $K \ge 2$ be the number of discrete prototypes and $D \ge 1$ be the embedding dimension. Let $\mathbf{C} \in \mathbb{R}^{K \times D}$ be the codebook matrix where the $k$-th \textbf{row} corresponds to the prototype $\mathbf{c}_k$.
We define the probability simplex $\Delta^{K-1} \coloneqq \{ \mathbf{p} \in \mathbb{R}^K : \mathbf{p} \succeq \mathbf{0}, \, \mathbf{1}^\top \mathbf{p} = 1 \}$.

The \textbf{soft-quantization map} $z: \Delta^{K-1} \to \mathbb{R}^D$ is defined as the expected embedding:
\begin{equation}
    z(\mathbf{p}) \coloneqq \mathbf{C}^\top \mathbf{p} = \sum_{k=1}^{K} p_k \mathbf{c}_k.
\end{equation}
We further define the \textbf{geometric radius} $M$ as the maximum Euclidean norm of any prototype...
\end{definition}

\paragraph{Lipschitz Bound.}

Training is performed in probability space (via KL divergence), whereas drift and anomaly scores are measured in embedding space.
We therefore need a simple link that converts discrepancies between assignment distributions into Euclidean changes after soft quantization.
Lemma~\ref{lem:lipschitz_app} provides exactly this bridge, with the codebook radius $M$ as the controlling constant.

\begin{lemma}[$\ell_1 \to \ell_2$ Lipschitzness]
\label{lem:lipschitz_app}
For any two assignment distributions $\mathbf{p}, \mathbf{q} \in \Delta^{K-1}$, the Euclidean distance between their soft-quantized embeddings is bounded by their $\ell_1$ distance, scaled by the codebook radius $M$:
\begin{equation}
    \|z(\mathbf{p}) - z(\mathbf{q})\|_2 \le M \|\mathbf{p} - \mathbf{q}\|_1.
\end{equation}
\end{lemma}
\begin{proof}
We expand the Euclidean distance and apply the triangle inequality.
\begin{align*}
    \|z(\mathbf{p}) - z(\mathbf{q})\|_2 
    &= \left\| \sum_{k=1}^K p_k \mathbf{c}_k - \sum_{k=1}^K q_k \mathbf{c}_k \right\|_2 \\
    &= \left\| \sum_{k=1}^K (p_k - q_k) \mathbf{c}_k \right\|_2 \\
    &\le \sum_{k=1}^K |p_k - q_k| \cdot \|\mathbf{c}_k\|_2 .
\end{align*}
Since $\|\mathbf{c}_k\|_2 \le M$ for all $k$, we have:
\[
    \sum_{k=1}^K |p_k - q_k| \cdot \|\mathbf{c}_k\|_2
    \le M \sum_{k=1}^K |p_k - q_k|
    = M \|\mathbf{p} - \mathbf{q}\|_1.
\]
\end{proof}

With this bridge in place, the remaining step is to control the probability-space differences themselves.
Predictive alignment bounds the deviation between targets and predictions at each time step, while EMA smoothness bounds how fast the target moves.
Combining these two controls yields the following stability bound.

\begin{theorem}[Stability Bound]
\label{thm:stability_app}
Let $\{\hat{\mathbf{p}}_t\}$ and $\{\mathbf{p}_t\}$ be the sequences of predicted and target (EMA) distributions, respectively, and let $\hat{\mathbf{z}}_t = z(\hat{\mathbf{p}}_t)$. Assume the following conditions hold for time steps $t$ and $t+1$:
\begin{enumerate}
    \item \textbf{Predictive Alignment:} $D_{\mathrm{KL}}(\mathbf{p}_t \| \hat{\mathbf{p}}_t) \le \varepsilon_t$ and $D_{\mathrm{KL}}(\mathbf{p}_{t+1} \| \hat{\mathbf{p}}_{t+1}) \le \varepsilon_{t+1}$.
    \item \textbf{Target Smoothness:} $\|\mathbf{p}_{t+1} - \mathbf{p}_t\|_1 \le \delta_t$.
\end{enumerate}
Then, the representation drift $\|\hat{\mathbf{z}}_{t+1} - \hat{\mathbf{z}}_t\|_2$ is bounded by:
\begin{equation}
    \|\hat{\mathbf{z}}_{t+1} - \hat{\mathbf{z}}_t\|_2 \le M \left( \sqrt{2\varepsilon_{t+1}} + \delta_t + \sqrt{2\varepsilon_t} \right).
\end{equation}
\end{theorem}
\begin{proof}
First, we recall \textbf{Pinsker's Inequality}~\cite{cover2006elements}, which provides a bound on the $\ell_1$ distance between two distributions in terms of their KL divergence:
\[
    \|\mathbf{p} - \mathbf{q}\|_1 \le \sqrt{2 D_{\mathrm{KL}}(\mathbf{p} \| \mathbf{q})}.
\]

We now decompose the drift term $\|\hat{\mathbf{z}}_{t+1} - \hat{\mathbf{z}}_t\|_2$ using the triangle inequality by introducing the intermediate target embeddings $\mathbf{z}_{t+1}$ and $\mathbf{z}_t$:
\begin{align*}
    \|\hat{\mathbf{z}}_{t+1} - \hat{\mathbf{z}}_t\|_2 
    &= \|\hat{\mathbf{z}}_{t+1} - \mathbf{z}_{t+1} + \mathbf{z}_{t+1} - \mathbf{z}_t + \mathbf{z}_t - \hat{\mathbf{z}}_t\|_2 \\
    &\le \underbrace{\|\hat{\mathbf{z}}_{t+1} - \mathbf{z}_{t+1}\|_2}_{\text{(A)}} 
    + \underbrace{\|\mathbf{z}_{t+1} - \mathbf{z}_t\|_2}_{\text{(B)}} 
    + \underbrace{\|\mathbf{z}_t - \hat{\mathbf{z}}_t\|_2}_{\text{(C)}}.
\end{align*}

\textbf{Bounding (A) and (C):} Using Lemma~\ref{lem:lipschitz_app} and Pinsker's Inequality:
\begin{align*}
    \|\hat{\mathbf{z}}_{t+1} - \mathbf{z}_{t+1}\|_2 
    &\le M \|\hat{\mathbf{p}}_{t+1} - \mathbf{p}_{t+1}\|_1 
    \le M \sqrt{2\varepsilon_{t+1}}.
\end{align*}
By symmetry, $\|\mathbf{z}_t - \hat{\mathbf{z}}_t\|_2 \le M \sqrt{2\varepsilon_t}$.

\textbf{Bounding (B):} Using Lemma~\ref{lem:lipschitz_app} and the smoothness assumption:
\[
    \|\mathbf{z}_{t+1} - \mathbf{z}_t\|_2 
    \le M \|\mathbf{p}_{t+1} - \mathbf{p}_t\|_1 
    \le M \delta_t.
\]

Summing these terms yields the final bound.
\end{proof}

\paragraph{Interpretation.}
This theorem implies that the sensitivity of our anomaly score is strictly controlled by the codebook radius $M$. Unlike unconstrained continuous models where gradients can cause unbounded output shifts (effectively $M \to \infty$), SC-JEPA guarantees that any significant representation drift arises from either a true shift in the target signal ($\delta_t$) or a high prediction error ($\varepsilon$). This effectively filters out false positives caused by trivial numerical perturbations.

\subsubsection{Non-Collapse Lower Bound}
\label{app:lower_bound}

\paragraph{Goal.}
We provide a simple \emph{non-collapse certificate} for the soft codebook.
Namely, if a mini-batch exhibits (i) non-trivial multi-code usage, (ii) sufficiently sharp assignments,
and (iii) non-negligible separation between two frequently used prototypes,
then the soft-quantized embeddings $\mathbf{z}_i=\mathbf{C}^\top\mathbf{p}_i$ have strictly positive batch variance,
i.e., $\operatorname{Tr}(\operatorname{Cov}(\{\mathbf{z}_i\}))>0$.

\paragraph{Setup.}
Consider a batch of assignment probabilities $\{\mathbf{p}_i\}_{i=1}^B \subset \Delta^{K-1}$ and their average
$\bar{\mathbf{p}} \coloneqq \frac{1}{B}\sum_{i=1}^B \mathbf{p}_i$.
Let $\mathbf{z}_i \coloneqq \mathbf{C}^\top \mathbf{p}_i \in \mathbb{R}^D$, where $\mathbf{C}\in\mathbb{R}^{K\times D}$ stacks prototypes as rows.

To preclude the constant-code corner case, we assume three mild properties.
Crucially, these are not arbitrary assumptions but are explicitly enforced by the optimization objectives defined in ~\Cref{sec:loss}:
\begin{itemize}[leftmargin=*, noitemsep, topsep=0pt]
    \item \textbf{Batch Entropy:} $\mathcal{L}_{\mathrm{ent}}^{\mathrm{batch}}$ maximizes $H(\bar{\mathbf{p}})$, enforcing Assumption~\ref{assum:batchent}.
    \item \textbf{Sharpness:} $\mathcal{L}_{\mathrm{ent}}^{\mathrm{sample}}$ minimizes individual entropy, driving $\mathbf{p}_i$ toward one-hot vectors (Assumption~\ref{assum:sharp2}).
    \item \textbf{Separation:} $\mathcal{L}_{\mathrm{com}}$ and EMA updates prevent prototype collapse (Assumption~\ref{assum:sep_mr}).
\end{itemize}

\begin{assumption}[Multi-code usage / non-trivial batch marginal]
\label{assum:batchent}
The batch marginal $\bar{\mathbf{p}}$ is not concentrated on a single index. Concretely,
\begin{equation}
H(\bar{\mathbf{p}})\ge \eta \quad \text{for some }\eta>0.
\end{equation}
\end{assumption}

\begin{assumption}[Assignment sharpness]
\label{assum:sharp2}
For each sample, let $k(i)=\arg\max_k p_{i,k}$.
There exists $\varepsilon\in[0,1)$ such that
\begin{equation}
\|\mathbf{p}_i-\mathbf{e}_{k(i)}\|_1 \le \varepsilon
\quad\text{for all } i .
\end{equation}
\end{assumption}

\begin{remark}
The deviation $\varepsilon$ is strictly controlled by the softmax temperature $\tau$. As $\tau \to 0$, $\mathbf{p}_i$ approaches a hard assignment, ensuring $\varepsilon \to 0$.
\end{remark}

\begin{assumption}[Separation for the two dominant codes]
\label{assum:sep_mr}
Define
\begin{equation}
\beta \coloneqq \frac{1-\rho^*}{K-1},
\end{equation}
where $\rho^* \in (1/K, 1)$ is the unique root of the equation $\mathcal{H}_b(\rho) + (1-\rho)\log(K-1) = \eta$.
Let $m=\arg\max_k \bar p_k$ and let $r\neq m$ be any index satisfying $\bar p_r\ge \beta$ (guaranteed by Lemma~\ref{lem:twoactive}).
Assume their prototypes satisfy
\begin{equation}
\|\mathbf{c}_m-\mathbf{c}_r\|_2 \ge \Delta_c
\quad\text{for some }\Delta_c>0.
\end{equation}
\end{assumption}

To drive a variance lower bound, we first (Lemma~\ref{lem:twoactive}) extract two non-negligible codes from entropy, then (Lemma~\ref{lem:freq}) convert their marginal mass into two dominant-code groups under sharp assignments.

\begin{lemma}[Entropy lower bound implies at least two active codes]
\label{lem:twoactive}
Let $m=\arg\max_k \bar p_k$ and $\bar p_{\max}=\bar p_m$.
If $H(\bar{\mathbf{p}})\ge \eta$, then $\bar p_{\max}\le \rho^*$,
where $\rho^*$ is defined in Assumption~\ref{assum:sep_mr}.
Moreover, there exists an index $r\neq m$ such that
\begin{equation}
\bar p_r \;\ge\; \beta \;\coloneqq\; \frac{1-\rho^*}{K-1}.
\end{equation}
\end{lemma}

\begin{proof}
Consider the function $\phi(x) \coloneqq \mathcal{H}_b(x) + (1-x)\log(K-1)$, which represents the maximum entropy achievable by a distribution with maximum probability mass $x$.
It is well-known (from Fano's inequality analysis) that $\phi(x)$ is strictly decreasing for $x \in [1/K, 1]$.
Since $H(\bar{\mathbf{p}}) \ge \eta$ and $H(\bar{\mathbf{p}}) \le \phi(\bar p_{\max})$, we have $\phi(\bar p_{\max}) \ge \eta = \phi(\rho^*)$.
By the monotonicity of $\phi$, this implies $\bar p_{\max} \le \rho^*$.
Consequently, the remaining mass is at least $1-\rho^*$. By the pigeonhole principle, at least one non-dominant index $r$ must satisfy $\bar p_r \ge (1-\rho^*)/(K-1) = \beta$.
\end{proof}

\begin{lemma}[Marginal mass transfers to dominant-code frequency]
\label{lem:freq}
Define $\hat{\pi}_k \coloneqq \frac{1}{B}\sum_{i=1}^B \mathbb{I}[k(i)=k]$.
Under Assumption~\ref{assum:sharp2}, we have $\|\bar{\mathbf{p}}-\hat{\boldsymbol{\pi}}\|_1 \le \varepsilon$,
hence for all $k$,
\begin{equation}
\hat{\pi}_k \ge \bar p_k - \varepsilon.
\end{equation}
\end{lemma}

\begin{proof}
\[
\|\bar{\mathbf{p}}-\hat{\boldsymbol{\pi}}\|_1
=
\left\|\frac{1}{B}\sum_{i=1}^B (\mathbf{p}_i-\mathbf{e}_{k(i)})\right\|_1
\le
\frac{1}{B}\sum_{i=1}^B \|\mathbf{p}_i-\mathbf{e}_{k(i)}\|_1
\le \varepsilon.
\]
The coordinate-wise bound follows from $|\bar p_k-\hat\pi_k|\le \|\bar{\mathbf{p}}-\hat{\boldsymbol{\pi}}\|_1$.
\end{proof}

We will use the following prototype-proximity bound implied by sharp assignments:
for any $i$,
\begin{equation}
\label{eq:proto_close_oneline}
\|\mathbf{z}_i-\mathbf{c}_{k(i)}\|_2
\le
\Big(\max_k\|\mathbf{c}_k\|_2\Big)\,\|\mathbf{p}_i-\mathbf{e}_{k(i)}\|_1
\le M\varepsilon,
\end{equation}
where $M\coloneqq\max_k\|\mathbf{c}_k\|_2$.

\begin{theorem}[Explicit variance lower bound / non-collapse]
\label{thm:var_lower2}
Under Assumptions~\ref{assum:batchent}--\ref{assum:sep_mr}, let $\beta$ be defined as in Assumption~\ref{assum:sep_mr} (i.e., $\beta = \frac{1-\rho^*}{K-1}$), and define
\[
\alpha \coloneqq \max\{\beta-\varepsilon,\,0\}.
\]

Let $(m,r)$ be the indices specified in Assumption~\ref{assum:sep_mr}.
If $\alpha>0$ and $2M\varepsilon<\Delta_c$, then the batch covariance of $\{\mathbf{z}_i\}_{i=1}^B$ obeys
\begin{equation}
\label{eq:var_lower}
\operatorname{Tr}\!\big(\operatorname{Cov}(\{\mathbf{z}_i\})\big)
\;\ge\;
\alpha^2\,(\Delta_c-2M\varepsilon)^2
\;>\;0.
\end{equation}
\end{theorem}

\begin{proof}
\emph{Step 1 (two codes have non-trivial dominant frequency).}
Let $A=\{i:\,k(i)=m\}$ and $R=\{i:\,k(i)=r\}$.
By Lemma~\ref{lem:twoactive}, $\bar p_r\ge \beta$. Since $m=\arg\max_k \bar p_k$, we also have $\bar p_m\ge \bar p_r\ge \beta$.
Applying Lemma~\ref{lem:freq} yields:
\begin{align*}
\Pr(i\in A) &= \hat\pi_m \ge \bar p_m-\varepsilon \ge \beta-\varepsilon=\alpha, \\[1ex]
\Pr(i\in R) &= \hat\pi_r \ge \bar p_r-\varepsilon \ge \beta-\varepsilon=\alpha.
\end{align*}

\emph{Step 2 (dominant-code membership implies proximity to the corresponding prototype).}
By \eqref{eq:proto_close_oneline} and Assumption~\ref{assum:sharp2}, we have tight concentration around prototypes:
\[
i \in A \implies \|\mathbf{z}_i-\mathbf{c}_m\|_2 \le M\varepsilon, \qquad
j \in R \implies \|\mathbf{z}_j-\mathbf{c}_r\|_2 \le M\varepsilon.
\]

\emph{Step 3 (two separated prototype-centered clusters imply positive covariance trace).}
For any $i\in A$ and $j\in R$, the triangle inequality gives:
\begin{align*}
\|\mathbf{z}_i-\mathbf{z}_j\|_2
&\ge
\|\mathbf{c}_m-\mathbf{c}_r\|_2 - \|\mathbf{z}_i-\mathbf{c}_m\|_2 - \|\mathbf{z}_j-\mathbf{c}_r\|_2 \\
&\ge
\Delta_c - 2M\varepsilon.
\end{align*}
Now view $\mathbf{Z}$ as a random variable obtained by uniformly sampling an index from the batch, and let $\mathbf{Z}'$ be an independent copy.
The pairwise-variance identity yields
$
\operatorname{Tr}(\operatorname{Cov}(\mathbf{Z}))
=
\frac{1}{2}\,\mathbb{E}\big[\|\mathbf{Z}-\mathbf{Z}'\|_2^2\big].
$
Restricting the expectation to the event $\{i\in A,\ j\in R\}$ yields the bound:
\begin{align*}
\operatorname{Tr}(\operatorname{Cov}(\mathbf{Z}))
&\ge
\frac{1}{2}\cdot 2\,\Pr(i\in A)\Pr(j\in R)\,(\Delta_c-2M\varepsilon)^2 \\[1ex]
&\ge
\alpha^2(\Delta_c-2M\varepsilon)^2.
\end{align*}
If $\alpha>0$ and $2M\varepsilon<\Delta_c$, the right-hand side is strictly positive, ruling out collapse.
\end{proof}

\paragraph{Remark.}
Theorem~\ref{thm:var_lower2} provides a certificate against collapse. The bound becomes vacuous only if
(i) $\alpha=0$ (low batch entropy), or (ii) $\Delta_c\le 2M\varepsilon$ (low sharpness or prototype separation).
However, these conditions are explicitly encouraged by the batch-entropy, sample-entropy, and prototype-alignment objectives introduced in Section~3.4.

% =========================================================
% Section A: Reproducibility Details
% =========================================================
\section{Reproducibility Details}
\label{app:reproducibility}

\subsection{Implementation Details}
\label{app:implementation}
\paragraph{Data Preprocessing.}
For each dataset, we identify and remove zero-variance channels using only the official training split, and apply the same channel mask to the corresponding test split. The remaining variables are normalized independently within each window using RevIN. We construct context--target pairs from consecutive non-overlapping windows of length $100$, where the current window $\mathbf{X}_t$ serves as input and the subsequent window $\mathbf{X}_{t+1}$ as the prediction target. Each window is partitioned into $P{=}5$ non-overlapping patches of length $L{=}20$. For SMD, each server machine is processed independently, and the resulting windows are aggregated for training and evaluation.

All methods are implemented in PyTorch and trained on a single NVIDIA GeForce RTX 4090 GPU. We use the Adam optimizer throughout and, unless otherwise specified, keep the reported hyperparameters fixed across datasets for controlled comparison. We set the latent dimension to $D{=}256$ and the codebook size to $K{=}128$, and maintain EMA target networks for both the encoder and codebook.

After each optimization step on the online parameters $\theta$, the corresponding target parameters $\xi$ are updated as
\[
\xi \leftarrow \rho \xi + (1-\rho)\theta,
\]
with decay $\rho{=}0.996$. Gradients are not propagated through the EMA target branch.

\textbf{Model Architectures.}
\begin{itemize}
    \item \textbf{Encoder:} Patch size $L{=}20$. The projection head is a 2-layer MLP (hidden $64{\to}32$) mapping to latent dimension $D{=}256$. The backbone is a 6-layer Transformer (8 heads, dropout $0.1$).
    \item \textbf{Codebook:} Codebook size $K{=}128$, dimension $D{=}256$. Soft quantization temperature $\tau{=}0.1$.
    \item \textbf{Predictors:} Both fine and coarse predictors are 2-layer Transformers (4 heads, hidden dimension $128$).
\end{itemize}

\textbf{Optimization \& Hyperparameters.}
\begin{itemize}
    \item \textbf{Training:} Adam optimizer (lr $5{\times}10^{-4}$, weight decay $10^{-5}$), batch size $128$. Target networks are updated via EMA (decay $0.996$). Gradient clipping norm is $0.5$. We train for a maximum of 100 epochs, initiating model selection after 50 epochs with an early stopping patience of 10.
   \item \textbf{Loss Weights:} Predicted objective weights $\lambda_f{=}1.0$, $\lambda_c{=}0.5$, and $\gamma{=}0.1$. Codebook alignment weights $\lambda_{\mathrm{emb}}{=}1.0$ and $\lambda_{\mathrm{com}}{=}0.25$. Entropy regularization weights $\lambda_{\mathrm{ent}}^{\mathrm{sample}}{=}0.005$ and $\lambda_{\mathrm{ent}}^{\mathrm{batch}}{=}0.01$. Reconstruction weight $\lambda_r$ is annealed linearly from $0.5$ to $0.1$ during training.
 
\end{itemize}

\subsection{SC-JEPA Pre-training Algorithm}
\label{app:pretrain}

We provide the detailed formulation of the total training objective for reproducibility.
Following Section~\ref{sec:loss}, the total loss is

\begin{align*}
\mathcal{L}
&=
\underbrace{
\lambda_f\!\left(
\mathcal{L}_{\mathrm{KL}}^{\mathrm{fine}}
+\gamma\mathcal{L}_{\mathrm{MSE}}^{\mathrm{fine}}
\right)
+\lambda_c\mathcal{L}_{\mathrm{KL}}^{\mathrm{coarse}}
}_{\text{Predictive Objective}}
\\
&\quad+
\underbrace{
\lambda_{\mathrm{emb}}\mathcal{L}_{\mathrm{emb}}
+\lambda_{\mathrm{com}}\mathcal{L}_{\mathrm{com}}
+\lambda_{\mathrm{ent}}^{\mathrm{sample}}
\mathcal{L}_{\mathrm{ent}}^{\mathrm{sample}}
-\lambda_{\mathrm{ent}}^{\mathrm{batch}}
\mathcal{L}_{\mathrm{ent}}^{\mathrm{batch}}
}_{\text{Codebook Objective}}
+
\underbrace{
\lambda_r\mathcal{L}_{\mathrm{rec}}
}_{\text{Reconstruction}}.
\end{align*}
Algorithm~\ref{alg:pretrain_app} summarizes the pre-training procedure.

\begin{algorithm}[H]
\caption{SC-JEPA Pre-training Procedure}
\label{alg:pretrain_app}
\begin{algorithmic}[1]
\REQUIRE Dataset $\mathcal{D}$, online parameters $\theta$, target parameters $\xi$, EMA decay $\rho$

\WHILE{not converged}
    \STATE Sample $(\mathbf{X}_t,\mathbf{X}_{t+1})\sim\mathcal{D}$
    \hfill $\triangleright$ context--future pair

    \STATE $\mathbf{X}^{\mathrm{coarse}}_{t+1}
    \leftarrow \operatorname{Downsample}(\mathbf{X}_{t+1})$
    \hfill $\triangleright$ coarse future view

    \STATE $\mathbf{h}_t\leftarrow E_\theta(\mathbf{X}_t)$
    \STATE $\mathbf{p}_t,\mathbf{z}_t\leftarrow Q_\theta(\mathbf{h}_t)$
    \hfill $\triangleright$ online soft codes

    \STATE $\hat{\mathbf{X}}_t\leftarrow D_\theta(\mathbf{z}_t)$
    \hfill $\triangleright$ reconstruction branch

    \STATE $\hat{\mathbf{p}}^{\mathrm{fine}}_{t+1},
    \hat{\mathbf{z}}_{t+1}
    \leftarrow P_\theta^{\mathrm{fine}}(\mathbf{p}_t)$
    \hfill $\triangleright$ fine prediction

    \STATE $\hat{\mathbf{p}}^{\mathrm{coarse}}_{t+1}
    \leftarrow P_\theta^{\mathrm{coarse}}(\mathbf{p}_t)$
    \hfill $\triangleright$ coarse prediction

    \STATE $\mathbf{h}^{\mathrm{fine}}_{t+1}
    \leftarrow E_\xi(\mathbf{X}_{t+1})$
    \STATE $\mathbf{p}^{\mathrm{fine}}_{t+1},
    \mathbf{z}_{t+1}
    \leftarrow Q_\xi(\mathbf{h}^{\mathrm{fine}}_{t+1})$
    \hfill $\triangleright$ stop-gradient fine target

    \STATE $\mathbf{h}^{\mathrm{coarse}}_{t+1}
    \leftarrow E_\xi(\mathbf{X}^{\mathrm{coarse}}_{t+1})$
    \STATE $\mathbf{p}^{\mathrm{coarse}}_{t+1}
    \leftarrow Q_\xi(\mathbf{h}^{\mathrm{coarse}}_{t+1})$
    \hfill $\triangleright$ stop-gradient coarse target

    \STATE Compute $\mathcal{L}_{\mathrm{pred}}$,
    $\mathcal{L}_{\mathrm{code}}$, and $\mathcal{L}_{\mathrm{rec}}$

    \STATE $\mathcal{L}_{\mathrm{total}}
    \leftarrow
    \mathcal{L}_{\mathrm{pred}}
    +\mathcal{L}_{\mathrm{code}}
    +\lambda_r\mathcal{L}_{\mathrm{rec}}$

    \STATE Update $\theta$ using $\mathcal{L}_{\mathrm{total}}$
    \hfill $\triangleright$ gradient update

    \STATE $\xi\leftarrow\rho\xi+(1-\rho)\theta$
    \hfill $\triangleright$ EMA target update
\ENDWHILE
\end{algorithmic}
\end{algorithm}

\subsection{Downstream Anomaly Prediction Protocol}
\label{app:downstream_protocol}
\textbf{Problem Formulation.}
We formulate the downstream task as Anomaly Prediction, distinguishing it from standard point-wise anomaly detection.
The objective is to utilize a historical context window $\mathbf{X}_t$ to predict the anomaly status of the subsequent non-overlapping future window $\mathbf{X}_{t+1}$.
This setup evaluates the model's ability to identify early warning signs (precursors) rather than merely reacting to ongoing failures.
Each dataset provides anomaly annotations as a binary indicator array aligned with the time axis,
denoted by $\mathbf{a}\in\{0,1\}^{T_{\mathrm{full}}}$.
Let $\mathcal{I}_{t+1}$ be the set of time indices covered by the future window $\mathbf{X}_{t+1}$.
We define the window-level label as
\begin{equation}
y_{t+1} \;=\; \mathbb{I}\!\left(\sum_{\tau \in \mathcal{I}_{t+1}} \mathbf{a}_{\tau} > 0\right) \in \{0,1\}.
\end{equation}
Let $\hat{s}_{t+1}\in[0,1]$ denote the predicted anomaly probability and let
\begin{equation}
\hat{y}_{t+1} \;=\; \mathbb{I}\!\left(\hat{s}_{t+1}>\delta^*\right)
\end{equation}
be the binary prediction using threshold $\delta^*$ selected on the validation set to maximize the window-level F1 score.

\textbf{Notation.}
We denote by $\mathbf{X}_t$ the context window and by $\mathbf{X}_{t+1}$ the subsequent non-overlapping future window.
The supervision signal is the \emph{future-window} label $y_{t+1}\in\{0,1\}$, which indicates whether $\mathbf{X}_{t+1}$ is anomalous under the window-level definition above.
Thus, each downstream training example is a pair $(\mathbf{X}_t, y_{t+1})$, and the goal is anomaly prediction.

We report window-level Precision, Recall, and F1-score:
\begin{equation}
\mathrm{Precision}=\frac{\mathrm{TP}}{\mathrm{TP}+\mathrm{FP}},\quad
\mathrm{Recall}=\frac{\mathrm{TP}}{\mathrm{TP}+\mathrm{FN}},\quad
\mathrm{F1}=\frac{2 \cdot \mathrm{Precision} \cdot \mathrm{Recall}}{\mathrm{Precision}+\mathrm{Recall}}.
\end{equation}
We additionally report the Area Under the Receiver Operating Characteristic curve (AUC) computed from the raw probabilities $\hat{s}_{t+1}$.

\textbf{Adaptation Strategy.}
We utilize the pre-trained SC-JEPA as a frozen feature extractor.
Specifically, we discard the projection heads and freeze the parameters of the encoder and codebook.
For a given input $\mathbf{X}_t$, the encoder yields soft code probabilities
$\mathbf{p}_t \in \mathbb{R}^{V \times P \times K}$, where $V$ denotes variables, $P$ patches, and $K$ the codebook size.
To capture the most salient anomalies across the multivariate dimensions while preserving temporal structure, we apply \textit{Variable-wise Max-Pooling}:
\begin{equation}
\mathbf{h}_t = \max_{v \in \{1 \dots V\}} \mathbf{p}_t^{(v)} \in \mathbb{R}^{P \times K}.
\end{equation}
The aggregated representation $\mathbf{h}_t$ is flattened and passed to an MLP classifier $C_\psi$, trained to estimate $\hat{s}_{t+1}=P(y_{t+1}=1 \mid \mathbf{h}_t)$.

\begin{algorithm}[H]
\caption{Downstream Anomaly Prediction Protocol}
\label{alg:downstream}
\begin{algorithmic}[1]
\REQUIRE Frozen encoder $E_\theta$, frozen codebook $Q$, and splits
$\mathcal{S}_{\mathrm{tr}}$, $\mathcal{S}_{\mathrm{val}}$, $\mathcal{S}_{\mathrm{te}}$
\STATE Initialize MLP classifier $C_\psi$

\FOR{each batch $(\mathbf{X}_t,y_{t+1})\in\mathcal{S}_{\mathrm{tr}}$}
    \STATE $\mathbf{p}_t \leftarrow Q(E_\theta(\mathbf{X}_t))$
    \hfill $\triangleright\ \mathbf{p}_t\in\mathbb{R}^{V\times P\times K}$
    \STATE $\mathbf{h}_t \leftarrow \max_{v}\mathbf{p}_t^{(v)}$
    \hfill $\triangleright\ \mathbf{h}_t\in\mathbb{R}^{P\times K}$
    \STATE $\hat{s}_{t+1}\leftarrow C_\psi(\mathbf{h}_t)$
    \STATE Update $\psi$ using $\operatorname{BCE}(\hat{s}_{t+1},y_{t+1})$
\ENDFOR

\STATE Select $\delta^*$ on $\mathcal{S}_{\mathrm{val}}$ to maximize window-level F1
\STATE Evaluate Precision, Recall, and F1 on $\mathcal{S}_{\mathrm{te}}$ using $\delta^*$, and AUC using $\hat{s}_{t+1}$
\end{algorithmic}
\end{algorithm}

% =========================================================
% Section C: Extended Experimental Results
% =========================================================
\section{Experimental Results}
\label{app:extended_results}

\subsection{Main Benchmark}
\label{app:main_table}
Table~\ref{tab:experiment} presents the comprehensive performance evaluation of our proposed method against baseline models across standard anomaly detection benchmarks: MSL, SMAP, SWaT, SMD and PSM. In addition to the F1-score and AUC reported in the main text, we provide detailed comparisons on Precision and Recall to offer a holistic view of the detection capabilities. All metrics are reported as the mean $\pm$ standard deviation over five independent runs. The results demonstrate that our method consistently achieves superior accuracy and robustness compared to both classical unsupervised baselines (e.g., K-Means, DeepSVDD) and recent state-of-the-art multivariate time-series models (e.g., PatchTST, TS2Vec).

\begin{table}[H]
\centering
\caption{Full results on five benchmark datasets. All metrics are reported in percentage (\%) as mean $\pm$ standard deviation over 5 runs with different random seeds. Best baselines are underlined.}
\label{tab:experiment}
\captionsetup{font=small}
\setlength{\tabcolsep}{2pt}
\renewcommand{\arraystretch}{1.2}
\resizebox{\textwidth}{!}{%
\begin{tabular}{l|cccc|cccc|cccc|cccc|cccc}
\toprule
\textbf{Models}
& \multicolumn{4}{c|}{\textbf{MSL}}
& \multicolumn{4}{c|}{\textbf{SMAP}}
& \multicolumn{4}{c|}{\textbf{SWaT}}
& \multicolumn{4}{c|}{\textbf{PSM}}
& \multicolumn{4}{c}{\textbf{SMD}} \\
\cline{2-21}
& F1 & AUC & Prec & Rec
& F1 & AUC & Prec & Rec
& F1 & AUC & Prec & Rec
& F1 & AUC & Prec & Rec
& F1 & AUC & Prec & Rec \\
\hline
K-Means
& 20.63 $\pm$ 2.30 & 52.17 $\pm$ 5.65 & 19.27 $\pm$ 7.16 & 31.34 $\pm$ 15.40
& 8.74 $\pm$ 3.39 & 39.82 $\pm$ 2.57 & 10.52 $\pm$ 8.17 & 10.16 $\pm$ 4.92
& 14.45 $\pm$ 0.38 & 62.20 $\pm$ 0.16 & 7.99 $\pm$ 0.32 & \underline{77.05} $\pm$ 8.16
& 42.75 $\pm$ 1.73 & 51.37 $\pm$ 3.80 & 30.40 $\pm$ 2.34 & 78.89 $\pm$ 21.18
& 14.79 $\pm$ 5.41 & 57.49 $\pm$ 0.22 & 35.83 $\pm$ 34.03 & 20.64 $\pm$ 11.95 \\

DeepSVDD
& 23.46 $\pm$ 0.98 & 53.67 $\pm$ 4.80 & 14.78 $\pm$ 2.14 & 73.87 $\pm$ 26.12
& 22.05 $\pm$ 2.97 & 44.07 $\pm$ 2.61 & 12.90 $\pm$ 1.19 & \underline{81.99} $\pm$ 24.23
& 15.28 $\pm$ 1.75 & 58.22 $\pm$ 2.34 & 12.65 $\pm$ 1.90 & 20.00 $\pm$ 3.34
& 44.45 $\pm$ 1.56 & 49.95 $\pm$ 3.10 & 30.31 $\pm$ 3.33 & 89.38 $\pm$ 16.00
& 9.41 $\pm$ 2.42 & 47.31 $\pm$ 4.44 & 6.25 $\pm$ 1.61 & 19.15 $\pm$ 4.85 \\

LSTM-VAE
& 22.42 $\pm$ 7.02 & 55.22 $\pm$ 10.27 & 20.84 $\pm$ 10.10 & 32.43 $\pm$ 15.96
& 22.77 $\pm$ 1.73 & 52.97 $\pm$ 2.60 & 14.35 $\pm$ 1.21 & 55.39 $\pm$ 4.47
& 54.60 $\pm$ 1.17 & 79.91 $\pm$ 1.00 & 64.92 $\pm$ 14.30 & 50.07 $\pm$ 8.30
& 56.73 $\pm$ 1.14 & 71.95 $\pm$ 0.50 & 49.32 $\pm$ 1.59 & 66.86 $\pm$ 2.10
& 10.04 $\pm$ 1.10 & 56.18 $\pm$ 0.86 & 5.78 $\pm$ 0.58 & 38.29 $\pm$ 6.92 \\

iTransformer
& 27.25 $\pm$ 2.03 & 64.61 $\pm$ 1.19 & 16.71 $\pm$ 1.25 & 75.40 $\pm$ 11.80
& \underline{33.00} $\pm$ 2.05 & 60.91 $\pm$ 1.58 & 22.67 $\pm$ 2.83 & 64.32 $\pm$ 10.77
& 70.49 $\pm$ 1.00 & 82.10 $\pm$ 0.36 & 98.17 $\pm$ 1.82 & 55.00 $\pm$ 1.18
& 54.12 $\pm$ 0.92 & 63.09 $\pm$ 0.75 & 37.37 $\pm$ 1.17 & 98.29 $\pm$ 2.10
& 14.20 $\pm$ 0.75 & 57.94 $\pm$ 0.85 & 8.11 $\pm$ 0.76 & \underline{62.34} $\pm$ 13.63 \\

TimesNet
& \underline{28.44} $\pm$ 5.79 & 59.70 $\pm$ 5.15 & 30.31 $\pm$ 7.40 & 36.59 $\pm$ 27.00
& 27.78 $\pm$ 1.48 & 57.01 $\pm$ 1.81 & 17.98 $\pm$ 0.60 & 62.81 $\pm$ 10.10
& 67.20 $\pm$ 3.94 & 83.24 $\pm$ 1.64 & 90.87 $\pm$ 15.87 & 54.42 $\pm$ 6.09
& 52.64 $\pm$ 0.02 & 56.12 $\pm$ 2.66 & 35.76 $\pm$ 0.09 & \underline{99.71} $\pm$ 0.57
& 12.51 $\pm$ 0.49 & 56.33 $\pm$ 0.87 & 7.64 $\pm$ 0.42 & 35.96 $\pm$ 6.64 \\
Qwen2.5-3B
& 21.24 $\pm$ 0.52 & 55.66 $\pm$ 0.84 & 17.82 $\pm$ 1.28 & 26.67 $\pm$ 2.22
& 0.34 $\pm$ 0.69 & 63.61 $\pm$ 1.05 & 1.67 $\pm$ 3.33 & 0.19 $\pm$ 0.38
& 17.29 $\pm$ 1.85 & 56.75 $\pm$ 7.35 & 41.82 $\pm$ 21.26 & 14.43 $\pm$ 6.83
& 47.44 $\pm$ 9.10 & 65.65 $\pm$ 5.93 & 50.44 $\pm$ 6.09 & 46.57 $\pm$ 14.36
& 13.00 $\pm$ 1.78 & 60.66 $\pm$ 4.59 & 8.29 $\pm$ 1.41 & 31.06 $\pm$ 3.65 \\
PAD
& 21.67 $\pm$ 0.03 & 55.97 $\pm$ 6.96 & 12.19 $\pm$ 0.06 & \underline{97.78} $\pm$ 4.44
& 26.92 $\pm$ 1.78 & 59.79 $\pm$ 0.30 & 18.15 $\pm$ 1.21 & 55.00 $\pm$ 11.80
& 69.75 $\pm$ 4.83 & 83.71 $\pm$ 0.61 & 81.54 $\pm$ 15.17 & 62.79 $\pm$ 4.32
& 57.82 $\pm$ 5.96 & 73.84 $\pm$ 1.85 & 57.54 $\pm$ 3.18 & 58.57 $\pm$ 9.17
& 12.52 $\pm$ 0.34 & 58.53 $\pm$ 2.35 & 6.71 $\pm$ 0.20 & 93.40 $\pm$ 2.46 \\

A2P
& 23.10 $\pm$ 0.85 & 59.43 $\pm$ 1.28 & 13.92 $\pm$ 1.18 & 74.12 $\pm$ 17.29
& 21.68 $\pm$ 4.10 & 60.37 $\pm$ 1.42 & 15.88 $\pm$ 1.26 & 38.85 $\pm$ 18.15
& 70.18 $\pm$ 0.26 & 81.92 $\pm$ 0.10 & \underline{99.10} $\pm$ 0.74 & 54.33 $\pm$ 0.33
& 50.71 $\pm$ 0.96 & 57.01 $\pm$ 0.38 & 37.16 $\pm$ 0.97 & 80.00 $\pm$ 3.48
& 15.21 $\pm$ 0.73 & 58.87 $\pm$ 1.55 & 12.23 $\pm$ 1.48 & 22.13 $\pm$ 6.33 \\
FCM 
& 24.57 $\pm$ 0.73 & 61.09 $\pm$ 0.76 & 14.25 $\pm$ 0.54 & 90.00 $\pm$ 7.37
& 30.60 $\pm$ 3.08 & 61.85 $\pm$ 0.86 & 25.25 $\pm$ 2.76 & 41.41 $\pm$ 9.44
& 69.09 $\pm$ 1.08 & 81.66 $\pm$ 0.37 & 94.25 $\pm$ 7.92 & 54.83 $\pm$ 1.86
& 52.34 $\pm$ 0.89 & 57.25 $\pm$ 1.18 & 37.00 $\pm$ 1.39 & 89.85 $\pm$ 4.37
&\underline{ 19.25} $\pm$ 1.46 & \underline{60.44} $\pm$ 0.65 & 15.59 $\pm$ 2.43 & 28.94 $\pm$ 10.43 \\
MICN
& 20.00 $\pm$ 0.96 & 60.18 $\pm$ 2.26 & 11.56 $\pm$ 0.57 & 75.29 $\pm$ 10.12
& 20.81 $\pm$ 2.38 & 51.16 $\pm$ 1.72 & 13.75 $\pm$ 3.01 & 55.58 $\pm$ 22.01
& 48.31 $\pm$ 3.19 & 79.38 $\pm$ 0.30 & 77.99 $\pm$ 7.96 & 35.33 $\pm$ 3.71
& 48.61 $\pm$ 1.38 & 54.97 $\pm$ 1.12 & 34.64 $\pm$ 1.33 & 82.09 $\pm$ 6.61
& 12.05 $\pm$ 0.70 & 52.21 $\pm$ 0.37 & 7.72 $\pm$ 0.53 & 28.72 $\pm$ 5.83 \\
Pathformer
& 21.69 $\pm$ 2.29 & 61.86 $\pm$ 0.66 & 13.14 $\pm$ 1.43 & 68.24 $\pm$ 18.45
& 23.58 $\pm$ 1.27 & 58.67 $\pm$ 0.38 & 14.10 $\pm$ 0.68 & 72.12 $\pm$ 7.37
& 14.90 $\pm$ 1.72 & 59.87 $\pm$ 1.24 & 9.20 $\pm$ 1.04 & 40.98 $\pm$ 9.16
& 51.77 $\pm$ 0.84 & 54.41 $\pm$ 3.09 & 35.04 $\pm$ 0.76 & 99.10 $\pm$ 1.19
& 14.32 $\pm$ 0.50 & 56.49 $\pm$ 1.71 & 9.41 $\pm$ 1.24 & 36.60 $\pm$ 15.82 \\
TS-JEPA
& 25.49 $\pm$ 4.22 & 60.33 $\pm$ 4.65 & 17.10 $\pm$ 3.91 & 56.84 $\pm$ 18.65
& 26.57 $\pm$ 1.67 & 57.38 $\pm$ 0.72 & 18.23 $\pm$ 0.74 & 49.73 $\pm$ 8.05
& \underline{71.95} $\pm$ 0.56 & 80.33 $\pm$ 0.69 & 84.32 $\pm$ 1.67 & 62.76 $\pm$ 0.33
& 53.32 $\pm$ 2.25 & 66.10 $\pm$ 3.11 & 36.95 $\pm$ 2.50 & 96.29 $\pm$ 2.65
& 7.49 $\pm$ 2.65 & 49.53 $\pm$ 8.44 & 19.11 $\pm$ 27.96 & 20.43 $\pm$ 14.37 \\

PatchTST
& 26.98 $\pm$ 0.51 & 60.43 $\pm$ 1.13 & \underline{75.00} $\pm$ 0.00 & 16.45 $\pm$ 0.38
& 30.06 $\pm$ 1.20 & \underline{61.62} $\pm$ 0.66 & 21.76 $\pm$ 0.93 & 49.06 $\pm$ 2.10
& 70.64 $\pm$ 3.57 & 81.93 $\pm$ 1.27 & 94.03 $\pm$ 4.25 & 57.58 $\pm$ 3.94
& \underline{58.17} $\pm$ 0.73 & \underline{75.76} $\pm$ 0.11 & \underline{64.68} $\pm$ 0.59 & 52.86 $\pm$ 0.90
& 11.79 $\pm$ 0.43 & 47.76 $\pm$ 0.95 & 6.58 $\pm$ 0.22 & 57.02 $\pm$ 3.40 \\

TS2Vec
& 23.48 $\pm$ 2.14 & \underline{64.86} $\pm$ 0.54 & 31.45 $\pm$ 2.46 & 18.82 $\pm$ 2.35
& 32.81 $\pm$ 0.32 & 61.48 $\pm$ 0.46 & \underline{23.81} $\pm$ 0.30 & 52.81 $\pm$ 1.95
& 67.00 $\pm$ 0.83 & \underline{83.76} $\pm$ 0.23 & 80.95 $\pm$ 2.10 & 57.17 $\pm$ 0.41
& 48.43 $\pm$ 1.01 & 72.13 $\pm$ 0.43 & 32.39 $\pm$ 0.58 & 95.96 $\pm$ 2.90
& 14.78 $\pm$ 1.46 & 56.28 $\pm$ 0.31 & 8.83 $\pm$ 1.28 & 50.43 $\pm$ 11.51 \\

\hline
\textbf{SC-JEPA}
& \textbf{33.58 $\pm$ 4.34} & \textbf{66.08 $\pm$ 3.25} & \textbf{35.87 $\pm$ 10.90} & \textbf{40.80 $\pm$ 15.68}
& \textbf{33.64 $\pm$ 1.45} & \textbf{65.41 $\pm$ 2.06} & \textbf{24.24 $\pm$ 1.28} & \textbf{56.02 $\pm$ 8.18}
& \textbf{72.89 $\pm$ 0.70} & \textbf{84.95 $\pm$ 0.82} & \textbf{98.00 $\pm$ 1.00} & \textbf{58.05 $\pm$ 1.20}
& \textbf{61.61 $\pm$ 4.32} & \textbf{77.85 $\pm$ 1.28} & \textbf{55.01 $\pm$ 2.47} & \textbf{72.00 $\pm$ 12.96}
& \textbf{20.82 $\pm$ 2.61} & \textbf{62.29 $\pm$ 0.90} & \textbf{18.18 $\pm$ 5.24} & 31.49 $\pm$ 13.48 \\
\bottomrule
\end{tabular}%
}
\end{table}

\subsection{Generality}
\label{app:generality}
We evaluate generality by comparing in-domain and cross-domain pre-training while keeping the downstream evaluation exactly the same.

For a target dataset, in-domain pre-training uses its official training set. Cross-domain pre-training uses the union of the other three datasets, with the target dataset completely excluded during pre-training. Since the source datasets may have different channel dimensions in this cross-domain setting, we treat each channel as an independent univariate series during pre-training, which avoids any explicit channel matching across datasets. After pre-training, we follow the same downstream protocol and evaluate on the same target test split in both settings, so the only difference is the source of pre-training data.
\begin{table}[H]
\centering
\caption{Comparison of in-domain versus cross-domain generalization.}
\label{tab:generalization}

\small
\setlength{\tabcolsep}{3.5pt}
\renewcommand{\arraystretch}{1.1}
\resizebox{\linewidth}{!}{
\begin{tabular}{c|l|cc|cc|cc|cc}
\toprule
\multirow{2}{*}{\textbf{Models}} & \multirow{2}{*}{\textbf{Settings}} 
& \multicolumn{2}{c|}{\textbf{MSL}} & \multicolumn{2}{c|}{\textbf{SMAP}} & \multicolumn{2}{c|}{\textbf{SWaT}} & \multicolumn{2}{c}{\textbf{PSM}} \\
\cline{3-10} 
& & \multicolumn{1}{c}{F1} & \multicolumn{1}{c|}{AUC}
& \multicolumn{1}{c}{F1} & \multicolumn{1}{c|}{AUC}
& \multicolumn{1}{c}{F1} & \multicolumn{1}{c|}{AUC}
& \multicolumn{1}{c}{F1} & \multicolumn{1}{c}{AUC} \\
\hline
\multirow{4}{*}{\textbf{PatchTST}} 
& In-Domain & 26.98$\pm$0.51 & 60.43$\pm$1.13 & 30.06$\pm$1.20 & 61.62$\pm$0.66 & 70.64$\pm$3.57 & 81.93$\pm$1.27 & 58.17$\pm$0.73 & 75.76$\pm$0.11 \\
& Cross-Domain & 27.08$\pm$0.44 & 60.80$\pm$1.17 & 26.21$\pm$0.24 & 58.77$\pm$1.63 & 54.34$\pm$5.96 & 68.76$\pm$0.87 & 50.07$\pm$0.19 & 53.72$\pm$0.93 \\
\arrayrulecolor{gray!80}\cline{2-10}\arrayrulecolor{black}
& Abs. Diff.   & +0.10 & +0.37 & -3.85 & -2.85 & -16.30 & -13.17 & -8.10 & -22.04 \\
& Rel. Diff.   & +0.37\% & +0.61\% & -12.81\% & -4.63\% & -23.07\% & -16.07\% & -13.92\% & -29.09\% \\

\hline
\multirow{4}{*}{\textbf{TS2Vec}} 
& In-Domain & 23.48$\pm$2.14 & 64.86$\pm$0.54 & 32.81$\pm$0.32 & 61.48$\pm$0.46 & 67.00$\pm$0.83 & 83.76$\pm$0.23 & 48.43$\pm$1.01 & 72.13$\pm$0.43 \\
& Cross-Domain & 22.87$\pm$2.67 & 62.00$\pm$0.74 & 26.99$\pm$0.75 & 57.54$\pm$0.35 & 65.08$\pm$0.53 & 82.91$\pm$0.02 & 48.15$\pm$7.73 & 56.22$\pm$3.49 \\
\arrayrulecolor{gray!80}\cline{2-10}\arrayrulecolor{black}
& Abs. Diff. & -0.61 & -2.86 & -5.82 & -3.94 & -1.92 & -0.85 & -0.28 & -15.91 \\
& Rel. Diff. & -2.60\% & -4.41\% & -17.74\% & -6.41\% & -2.87\% & -1.01\% & -0.58\% & -22.06\% \\
\hline
\multirow{4}{*}{\textbf{SC-JEPA}} 
& In-Domain & 33.58$\pm$4.34 & 66.08$\pm$3.25 & 33.64$\pm$1.45 & 65.41$\pm$2.06 & 72.89$\pm$0.70 & 84.95$\pm$0.82 & 61.61$\pm$4.32 & 77.85$\pm$1.28 \\
& Cross-Domain & 33.48$\pm$7.26 & 67.13$\pm$6.15 & 33.81$\pm$5.13 & 68.88$\pm$1.45 & 71.10$\pm$0.76 & 82.96$\pm$0.40 & 56.87$\pm$4.31 & 66.80$\pm$6.02  \\
\arrayrulecolor{gray!80}\cline{2-10}\arrayrulecolor{black}
& Abs. Diff. & -0.10 & +1.05 & +0.17 & +3.47 & -1.79 & -1.99 & -4.74 & -11.05 \\
& Rel. Diff. & -0.30\% & \textbf{+1.59\%} & +0.51\% & \textbf{+5.30\%} & -2.46\% & \textbf{-2.34\%} & -7.69\% & \textbf{-14.19\%} \\
\bottomrule
\end{tabular}
}

\end{table}

\subsection{Ablation Studies}
\label{app:ablation}

Table~\ref{tab:ablation} reports the complete ablation results. All variants follow the same training and evaluation protocol as the full model, and we modify only the targeted component while keeping the rest of the architecture and training pipeline unchanged. For the codebook operation studies, the full model uses a discrete codebook bottleneck that represents each latent via soft assignment to a finite set of prototypes, with a slowly updated EMA target branch providing a stable reference for prototype states and assignment targets. To separate the effect of the bottleneck structure from the effect of auxiliary constraints, we evaluate two complementary interventions. In the loss-only removal setting, we keep the soft codebook bottleneck in both the online and EMA target branches, and we set the auxiliary codebook regularizer weights to zero. In the module removal setting, we remove the codebook bottleneck from both branches. The model then operates purely in continuous latent space, and the predictive objective reduces to an MSE loss because the code-assignment KL term is no longer defined without discrete code distributions.
\paragraph{Note.}
When the standard deviation across random seeds is (near) zero while the mean score is low, this typically indicates model collapse, i.e., the model makes nearly identical (often single-class) predictions across runs and loses discriminative signal.
\begin{table}[H]
\centering
\caption{Ablation results on five benchmark datasets.}
\label{tab:ablation}
\setlength{\tabcolsep}{3pt}
\renewcommand{\arraystretch}{1.15}
\captionsetup{font=small}
\resizebox{\textwidth}{!}{
\begin{tabular}{l|cc|cc|cc|cc|cc}
\toprule
\textbf{Datasets} 
& \multicolumn{2}{c|}{\textbf{MSL}} 
& \multicolumn{2}{c|}{\textbf{SMAP}} 
& \multicolumn{2}{c|}{\textbf{SWaT}} 
& \multicolumn{2}{c|}{\textbf{PSM}}
& \multicolumn{2}{c}{\textbf{SMD}} \\
\hline
\textbf{Metrics} 
& F1 & AUC & F1 & AUC & F1 & AUC & F1 & AUC & F1 & AUC \\
\hline

w/o KL Regularization      
& 29.01 $\pm$ 6.64 & 53.90 $\pm$ 5.25
& 28.71 $\pm$ 1.71 & 60.19 $\pm$ 0.95
& 67.33 $\pm$ 3.97 & 79.95 $\pm$ 0.97
& 54.93 $\pm$ 2.54 & 70.83 $\pm$ 1.25
& 12.63 $\pm$ 3.75 & 53.61 $\pm$ 8.33 \\

w/o Reconstruction Decoder 
& 25.96 $\pm$ 1.82 & 52.68 $\pm$ 1.68
& 25.15 $\pm$ 0.31 & 52.80 $\pm$ 1.88
& 14.77 $\pm$ 1.42 & 53.30 $\pm$ 2.45
& 53.03 $\pm$ 0.00 & 69.84 $\pm$ 7.74
& 10.18 $\pm$ 1.65 & 50.44 $\pm$ 1.44 \\

w/o Predictive Objective
& 28.88 $\pm$ 6.41 & 53.17 $\pm$ 1.03
& 27.05 $\pm$ 2.31 & 59.81 $\pm$ 1.76
& 59.47 $\pm$ 1.55 & 78.60 $\pm$ 0.59
& 58.89 $\pm$ 5.50 & 70.17 $\pm$ 4.73
& 10.04 $\pm$ 3.42 & 50.24 $\pm$ 5.07 \\

w/o Codebook Loss
& 31.62 $\pm$ 1.71 & 58.93 $\pm$ 1.81
& 31.53 $\pm$ 7.88 & 62.06 $\pm$ 0.60
& 72.64 $\pm$ 1.24 & 82.63 $\pm$ 0.63
& 60.83 $\pm$ 4.09 & 75.85 $\pm$ 0.65
& 12.82 $\pm$ 2.01 & 57.36 $\pm$ 4.25 \\

w/o Codebook Module  
& 21.82 $\pm$ 0.00 & 43.02 $\pm$ 9.30
& 21.69 $\pm$ 0.00 & 51.00 $\pm$ 1.23 
& 11.51 $\pm$ 0.00 & 50.00 $\pm$ 0.00 
& 53.03 $\pm$ 0.00 & 46.61 $\pm$ 6.79
& 12.45 $\pm$ 0.00 & 50.00 $\pm$ 0.00 \\

w/o Downsampling        
& 28.77 $\pm$ 6.95 & 63.16 $\pm$ 3.83
& 30.39 $\pm$ 2.47 & 61.34 $\pm$ 0.15
& 71.62 $\pm$ 0.28 & 83.60 $\pm$ 0.40
& 60.79 $\pm$ 3.31 & 71.72 $\pm$ 4.68
& 15.50 $\pm$ 1.54 & 60.93 $\pm$ 0.10 \\

\textbf{Full Model} 
& \textbf{33.58 $\pm$ 4.34} & \textbf{66.08 $\pm$ 3.25}
& \textbf{33.64 $\pm$ 1.45} & \textbf{65.41 $\pm$ 2.06}
& \textbf{72.89 $\pm$ 0.70} & \textbf{84.95 $\pm$ 0.82}
& \textbf{61.61 $\pm$ 4.32} & \textbf{77.85 $\pm$ 1.28}
& \textbf{20.82 $\pm$ 2.61} & \textbf{62.29 $\pm$ 0.90} \\

\bottomrule
\end{tabular}
}
\end{table}

% =========================================================
% Section D: Visualizations and Qualitative Analysis
% =========================================================
\section{Visualizations and Qualitative Analysis}
\label{app:viz}

\subsection{Framework}
\label{app:framework}

Figure~\ref{fig:framework} illustrates the end-to-end data flow and objectives of our framework, which consists of five coordinated components:
\begin{figure}[H]
    \centering
    \includegraphics[width=\linewidth]{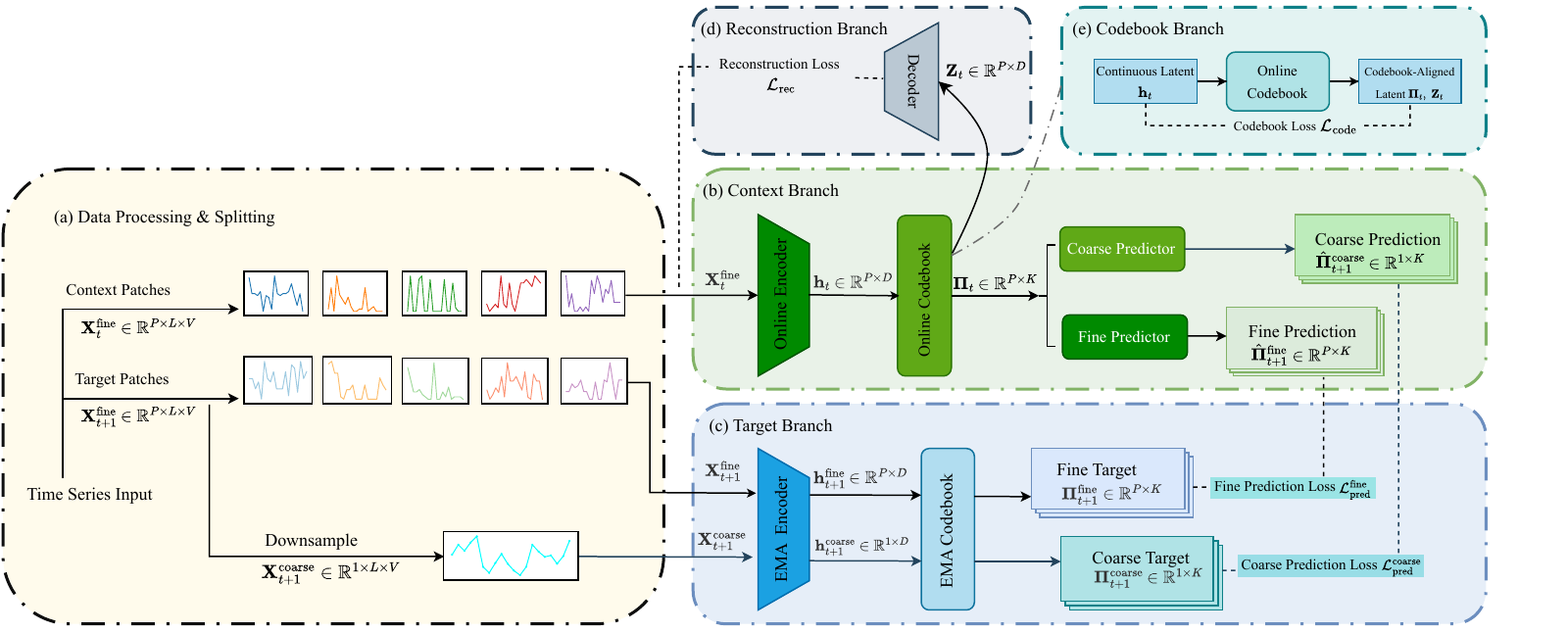}
    \caption{Overview of the Proposed Framework.}
    \label{fig:framework}
\end{figure}

\begin{enumerate}[label=(\alph*), leftmargin=*]
    \item \textbf{Data Processing \& Splitting:} We construct consecutive context--target pairs from non-overlapping windows. The context window is tokenized into fine-grained patches $\mathbf{X}^{\text{fine}}_t \in \mathbb{R}^{P \times L \times V}$. Simultaneously, the future window is processed into two views: fine-grained patches $\mathbf{X}^{\text{fine}}_{t+1}$ and a downsampled coarse view $\mathbf{X}^{\text{coarse}}_{t+1}$, obtained by averaging every $P$ consecutive time points to capture low-frequency trends.

    \item \textbf{Context Branch:} The online encoder processes the context patches $\mathbf{X}^{\text{fine}}_t$ to produce continuous latents $\mathbf{h}_t$. These are passed through the Online Codebook to generate the discrete code distributions $\boldsymbol{\Pi}_t \in \mathbb{R}^{P \times K}$. This sequence $\boldsymbol{\Pi}_t$ is then shared by two predictors: the \textit{Fine Predictor} outputs patch-level predictions $\hat{\boldsymbol{\Pi}}^{\text{fine}}_{t+1}$, while the \textit{Coarse Predictor} aggregates the history to produce a global prediction $\hat{\boldsymbol{\Pi}}^{\text{coarse}}_{t+1}$.

    \item \textbf{Target Branch:} To provide stable supervision, the momentum-updated EMA encoder processes both future views ($\mathbf{X}^{\text{fine}}_{t+1}$ and $\mathbf{X}^{\text{coarse}}_{t+1}$). The resulting features are mapped by the EMA Codebook into target distributions $\boldsymbol{\Pi}^{\text{fine}}_{t+1}$ and $\boldsymbol{\Pi}^{\text{coarse}}_{t+1}$. The model minimizes the prediction losses $\mathcal{L}^{\text{fine}}_{\text{pred}}$ and $\mathcal{L}^{\text{coarse}}_{\text{pred}}$ by matching the online predictions to these targets.

    \item \textbf{Reconstruction Branch:} To prevent representation collapse and anchor the semantics, a decoder reconstructs the original signal from the codebook-aligned embeddings $\mathbf{Z}_t$, optimizing the reconstruction loss $\mathcal{L}_{\text{rec}}$.

    \item \textbf{Codebook Branch:} This module aligns the continuous latents $\mathbf{h}_t$ with the discrete code assignments $\boldsymbol{\Pi}_t$ and embeddings $\mathbf{Z}_t$. It computes the codebook loss $\mathcal{L}_{\text{code}}$, which regularizes the assignment probabilities to ensure the codebook remains utilized and structured.
\end{enumerate}

\subsection{Inference Efficiency}
\label{app:inference}
\begin{figure}[H]
    \centering
    % Image 1
    \begin{subfigure}{0.32\linewidth}
        \centering
        \includegraphics[width=\linewidth]{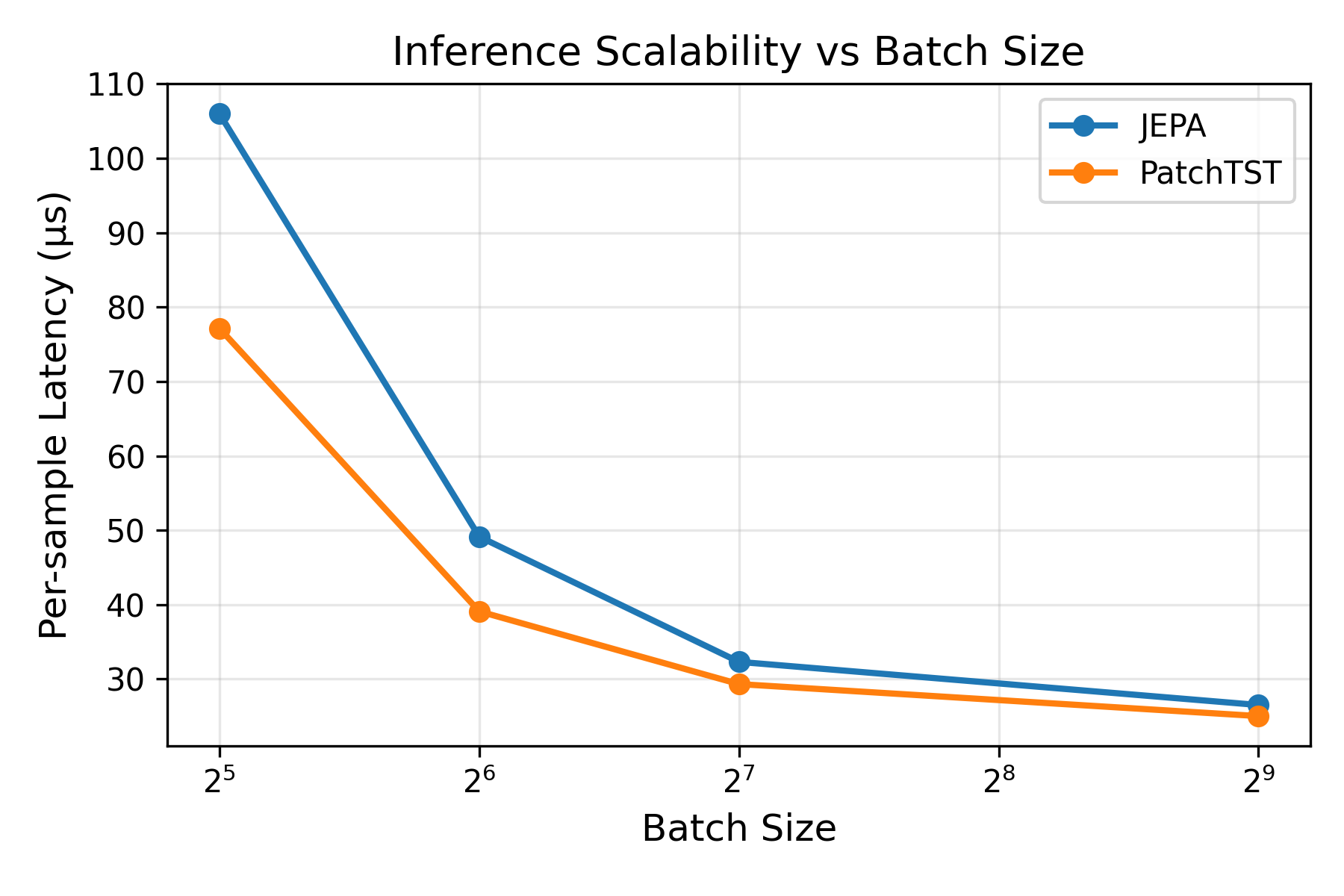}
        \caption{Latency vs. Batch}
        \label{fig:batch_scaling}
    \end{subfigure}
    \hfill % Pushes the next image to the right
    % Image 2
    \begin{subfigure}{0.32\linewidth}
        \centering
        \includegraphics[width=\linewidth]{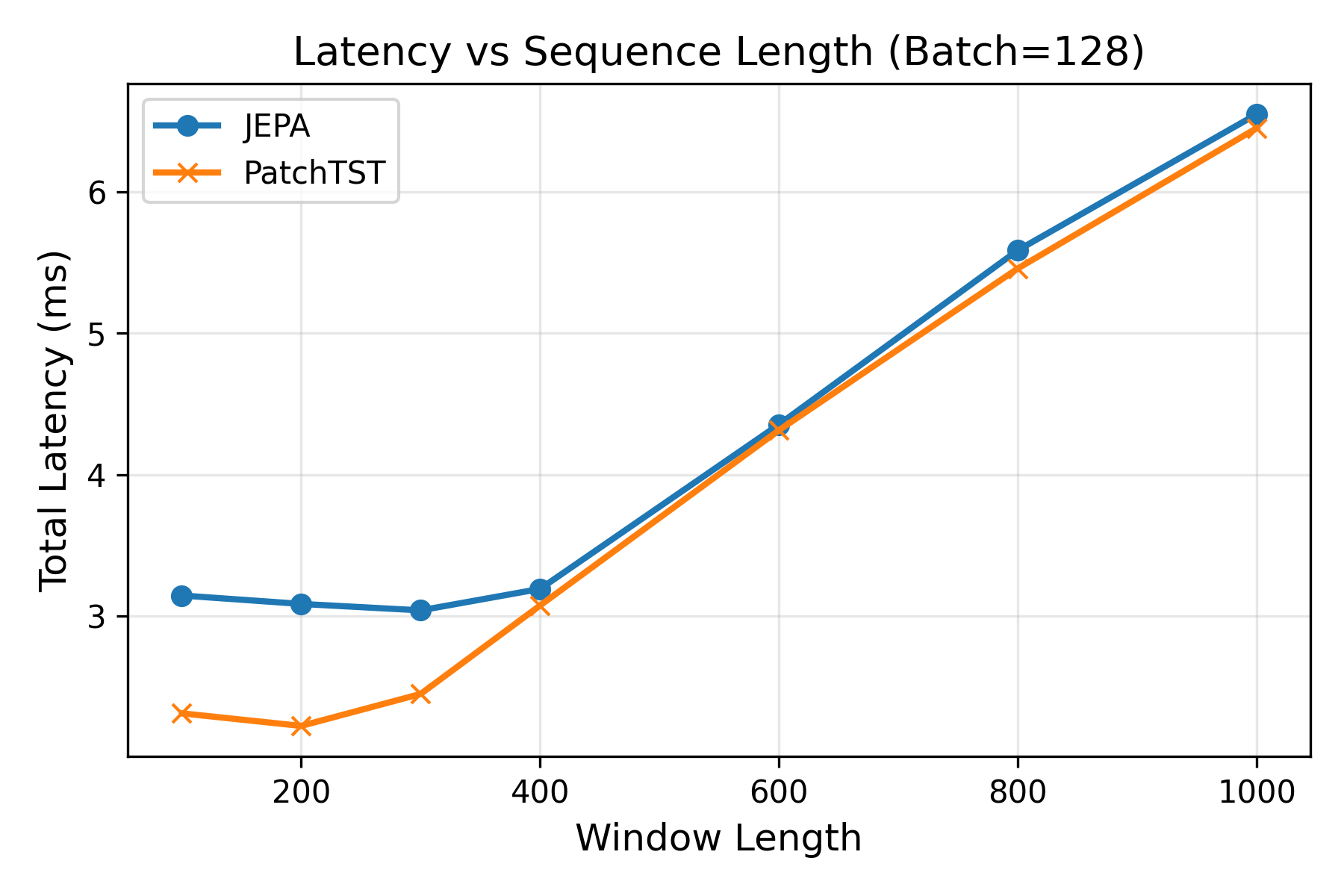}
        \caption{Latency vs. Window}
        \label{fig:window_scaling}
    \end{subfigure}
    \hfill % Pushes the next image to the right
    % Image 3
    \begin{subfigure}{0.32\linewidth}
        \centering
        \includegraphics[width=\linewidth]{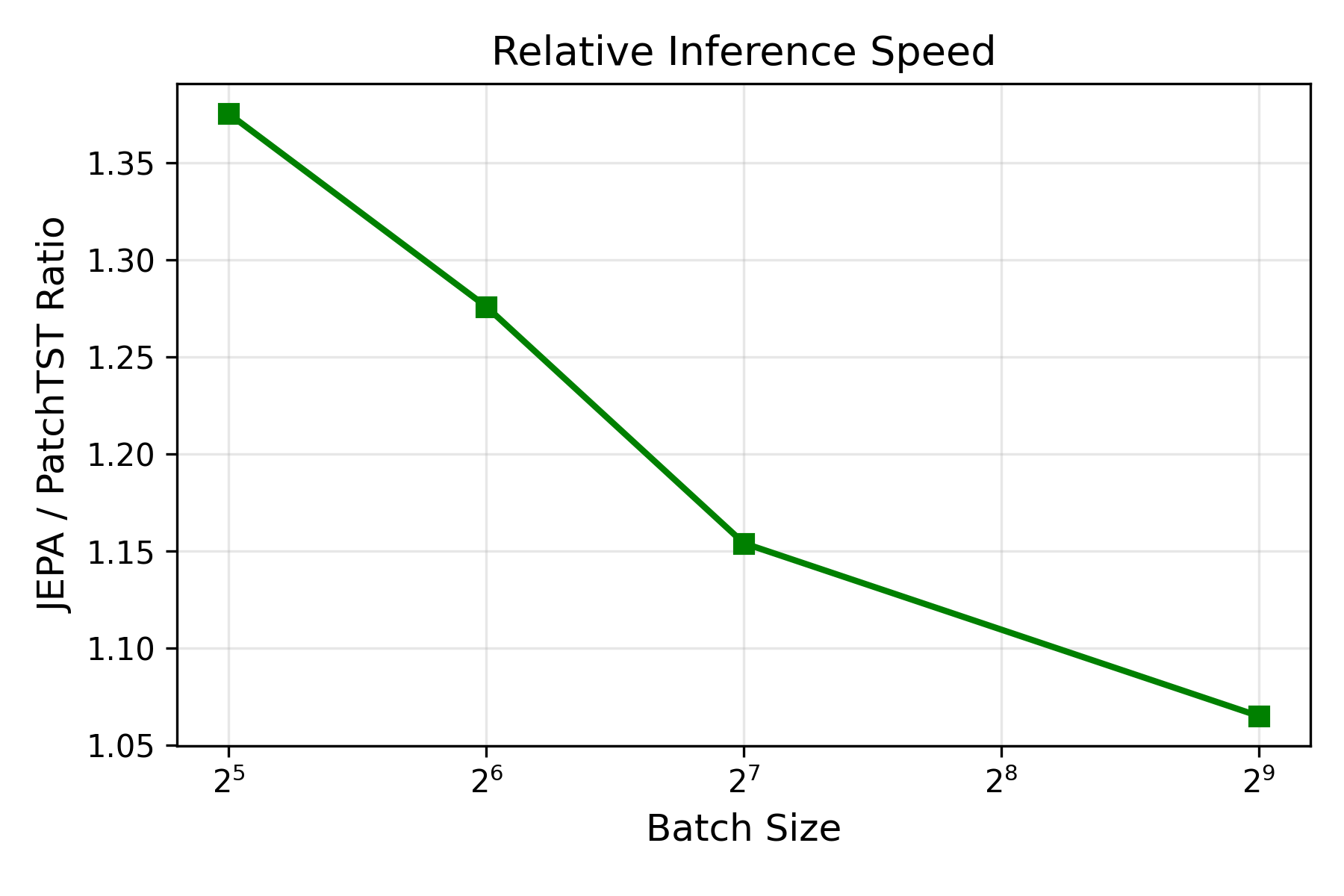}
        \caption{Latency Ratio}
        \label{fig:ratio_curve}
    \end{subfigure}
    
    \caption{Inference-time scaling comparisons across batch size, window length, and latency ratios.}
    \label{fig:inf}
\end{figure}
Figure~\ref{fig:inf} evaluates the inference efficiency of SC-JEPA compared with PatchTST from three complementary perspectives.
First, Figure~\ref{fig:batch_scaling} reports the per-sample latency vs.\ batch size: SC-JEPA is slower, but the gap shrinks at larger batches. Figure~\ref{fig:window_scaling} then reports latency vs.\ window length (batch fixed): both scale similarly, with SC-JEPA slightly slower. Figure~\ref{fig:ratio_curve} confirms this trend: the latency ratio decreases as batch size increases.

Overall, these results highlight a limitation of our approach: inference efficiency is not a primary contribution of SC-JEPA, as its richer architecture incurs higher computational cost than a single-backbone baseline.
Our design instead prioritizes early-warning performance and representation stability, with the runtime overhead becoming more manageable in deployment scenarios that support moderate to large batch sizes.

\end{document}